\documentclass[11pt]{article}

\usepackage[preprint]{acl}

\usepackage{times}
\usepackage{latexsym}

\usepackage[T1]{fontenc}

\usepackage[utf8]{inputenc}

\usepackage{microtype}

\usepackage{inconsolata}

\usepackage{graphicx}
\usepackage{multirow}
\usepackage{microtype}
\usepackage{graphicx}
\usepackage{amsmath,amssymb}
\usepackage{amsthm}
\usepackage{xcolor}
\usepackage{colortbl} 
\usepackage{array} 
\usepackage{booktabs} 
\usepackage{subcaption}

\theoremstyle{plain}
\newtheorem{theorem}{Theorem}[section]
\newtheorem{proposition}[theorem]{Proposition}
\newtheorem{lemma}[theorem]{Lemma}
\newtheorem{corollary}[theorem]{Corollary}
\theoremstyle{definition}
\newtheorem{definition}[theorem]{Definition}

\theoremstyle{remark}
\newtheorem{remark}[theorem]{Remark}

\title{$\mathcal{R}\mathcal{I}\mathcal{N}\mathcal{G}$: $\mathcal{R}$etrieval-$\mathcal{I}$nternalized $\mathcal{G}$eneration for Continual Large-Scale Knowledge Injection}

\newcommand{\pgen}{p_{\theta_{\mathrm{gen}}}}
\newcommand{\Ezp}{\mathbb{E}_{z\sim p_\phi(\cdot\mid x)}}
\newcommand{\Wkd}{W^{(\ell)}_{\text{knw},\text{down}}}
\newcommand{\Wkg}{W^{(\ell)}_{\text{knw},\text{gate}}}
\newcommand{\Wku}{W^{(\ell)}_{\text{knw},\text{up}}}

\author{%
Shicheng Xu$^{1,2}$ \quad Liang Pang$^{1}$\thanks{\ \ Corresponding authors.} \quad Liyi Chen$^{3}$ \quad Zihao Wei$^{1,2}$ \quad Jingcheng Deng$^{1,2}$ \\ \bf Yan Gao$^{3}$ \quad Yi Wu$^{3}$ \quad Yao Hu$^{3}$ \quad Huawei Shen$^{1}$  \quad Xueqi Cheng$^{1}$\\
$^{1}$State Key Laboratory of AI Safety,
 Institute of Computing Technology, CAS \\
 $^{2}$University of Chinese Academy of Sciences \quad $^{3}$Xiaohongshu Inc.\\
  \small{\texttt{xushicheng21s@ict.ac.cn} \quad\texttt{pangliang@ict.ac.cn}}}

\begin{document}
\maketitle
\begin{abstract}
Retrieval-augmented generation (RAG) improves factuality but adds latency and engineering overhead at serving time. We propose \textbf{RING} (\textit{Retrieval-Internalized Generation}), a {holistic paradigm spanning both architecture and training} that injects large-scale external knowledge into a \textit{Mixture-of-Memory Experts} and learns \textit{parametric search} over this internal memory via reinforcement learning, removing the external retriever entirely. Training proceeds in three stages: continued pre-training injects new corpora into a Knowledge Expert via our novel \textit{Dual Causal Attention}; supervised fine-tuning teaches a ``search-then-answer'' pattern; and reinforcement learning with hierarchical rewards optimizes the routing-and-search policy over the parametric memory. Unlike prior parametric injection methods that pair internal memory with a {fixed} or {rule-based} retriever, RING {learns} its retrieval policy directly from task signals. We further frame RING theoretically as a search-free approximation to the classical RAG objective. To evaluate large-scale injection of genuinely {new} knowledge without test-time leakage, we further construct \textbf{News-2025}, a benchmark built from news strictly post-dating the base LLM's pretraining cutoff. RING matches or surpasses both search-based RAG and parametric injection baselines in accuracy and efficiency.
\end{abstract}

\section{Introduction}
Large language models (LLMs) store factual knowledge in their parameters~\cite{hu2024towards} but struggle to assimilate new knowledge reliably, leading to hallucinations and outdated outputs~\cite{wang2023survey}. Retrieval-Augmented Generation (RAG) addresses this by bridging LLMs with external knowledge sources~\cite{lewis2020retrieval,xu2024search}: relevant documents are retrieved and appended to the LLM's context, grounding generation in external evidence~\cite{gao2023retrieval,wei2026dynamic,deng2025latent}.

Despite RAG's success, its reliance on external retrievers introduces significant inference latency and engineering complexity, and even generative retrieval~\cite{li2025matching} still depends on an external corpus for content mapping. Research has thus pivoted toward purely parametric injection, but existing approaches face severe bottlenecks: Knowledge Editing~\cite{zhang2024comprehensive} and Continual Learning~\cite{wang2024comprehensive} struggle with unstructured updates~\cite{deng2024everything} and catastrophic forgetting~\cite{nguyen2019toward}; Modular Injection methods (KBLAM~\cite{wang2025kblam}, SR-KI~\cite{yu2025sr}) require memory linear in corpus size; and Parametric RAG~\cite{su2025parametric} still depends on an external search engine to identify relevant documents at inference. {A scalable, fully parametric solution that matches RAG's performance remains elusive}.

\begin{figure*}[t]
\begin{center}
\centerline{\includegraphics[width=0.9\linewidth]{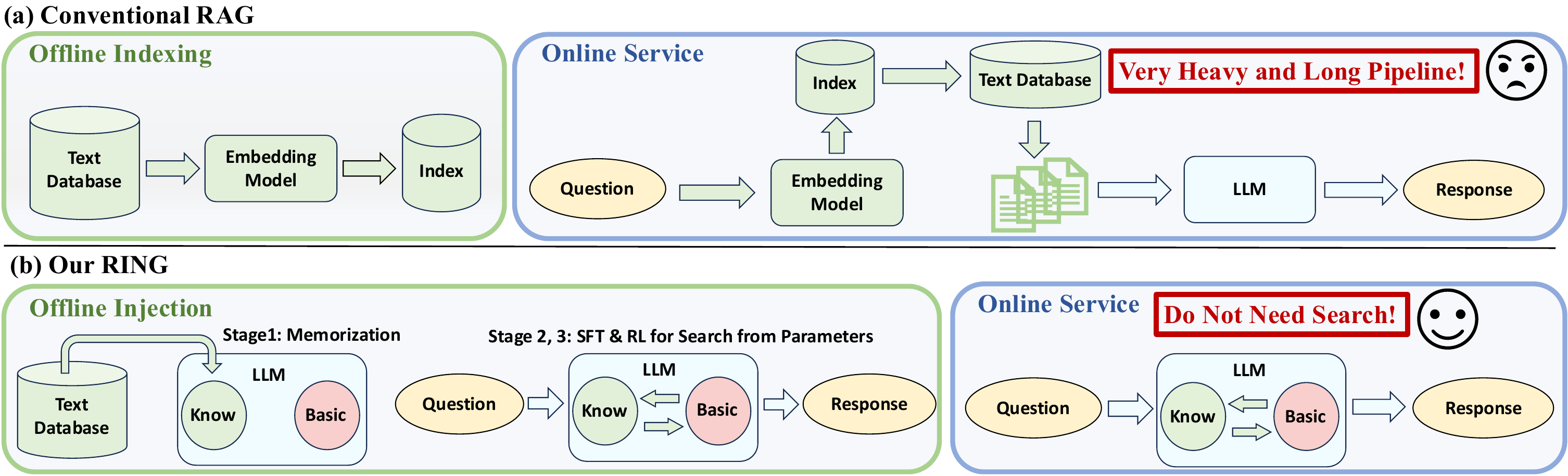}}
\caption{Overview of RING versus Conventional RAG. Unlike traditional methods that depend on external embedding models and vector databases (top), RING (bottom) unifies indexing and retrieval within the LLM parameters via a Mixture-of-Experts architecture. This design shifts retrieval computation to the training stage, achieving high-efficiency search-free inference for RAG.}
\label{overview}
\end{center}
\end{figure*}
To address this gap, we propose to {internalize} external search within the LLM parameters and call this \emph{holistic paradigm spanning both architecture and training} \textbf{Retrieval-Internalized Generation (RING)}, in direct contrast to RAG: where RAG {augments} the LLM with an external retriever, RING {internalizes} the retriever inside the LLM as a {learned policy} via the joint co-design of architecture and training. The central thesis of RING is that the bottleneck of parametric knowledge injection is not how to {store} knowledge in weights, but how to {search} those weights at inference time. Prior parametric methods all pair internal memory with a {fixed} or {rule-based} retrieval mechanism---attention similarity (KBLAM, SR-KI), $k$NN imitation (MLP Memory), hierarchical pruning (AtlasKV), or unsupervised routing (LAG)---and consequently underperform RAG at scale. {RING is the first to learn the parametric retrieval policy end-to-end via reinforcement learning}: hierarchical rewards over routing accuracy, search relevance, and answer correctness train the model to jointly select which expert holds the relevant knowledge and which sub-region inside that expert to attend to. This transforms parametric retrieval from a static lookup into a {trainable behavior}, eliminating the external index entirely while yielding cost profiles that scale with model FLOPs rather than index size.

\textbf{Architecture.} RING augments a dense LLM with a sparse \emph{Mixture-of-Memory Experts}: a \emph{Basic Expert} preserves the original LLM weights to mitigate catastrophic forgetting; a \emph{Knowledge Expert} parametrically memorizes and indexes the new corpus; and a learnable \emph{Router} selects between them at each token, integrating newly injected knowledge with the LLM's base reasoning.

\textbf{Three-Stage Training.} RING operationalizes this architecture with a CPT--SFT--RL pipeline. {(i) Continued Pre-training} injects the new corpus into the Knowledge Expert via our novel {Dual Causal Attention (DCA)}, which models new passages bidirectionally while preserving the LLM's causal generation~\cite{zhang2025bidirectional}. {(ii) Supervised Fine-Tuning} teaches a \emph{search-then-answer} response pattern within a single generative pass, with no external retriever. {(iii) Reinforcement Learning is the core of our pipeline and the locus of RING's main innovation.} Drawing on evidence that RL improves an LLM's ability to sample useful trajectories from its own parameter space~\cite{yue2025does}, we cast parametric retrieval as a policy optimization problem and use hierarchical rewards to simultaneously train (a) the Router to select the right expert, (b) the Knowledge Expert to surface the relevant memory fragment, and (c) the generator to answer grounded in that fragment. The result is a fully {learned} parametric retriever---unlike all prior parametric injection methods whose retrieval mechanism is hand-designed and frozen.

\textbf{Theoretical Grounding.} We further formulate RING as a discrete latent variable model and prove that the Knowledge Expert acts as a differentiable Key--Value memory index, while our CPT--SFT--RL pipeline performs variational inference that maximizes the Evidence Lower Bound, theoretically guaranteeing RING's capacity to approximate the classical RAG objective.

\textbf{News-2025 Benchmark.} To rigorously evaluate large-scale injection of {genuinely new} knowledge without test-time leakage, we construct {News-2025}, a bilingual benchmark of $\sim$170k Chinese / English news documents (341M tokens) strictly post-dating the base LLM's pretraining cutoff, so that correct answers cannot come from pre-existing knowledge or associative reasoning. Our contributions are:
\begin{itemize}
\vspace{-0.8em}
    \item We propose Retrieval-Internalized Generation (RING), a {holistic paradigm spanning both architecture and training} that eliminates external search by internalizing both indexing and retrieval into LLM parameters.
    \vspace{-0.8em}
    \item We construct {News-2025}, a bilingual post-cutoff benchmark for contamination-free evaluation of knowledge injection at scale.
    \vspace{-0.8em}
    \item Empirically, RING outperforms parametric injection baselines on News-2025, matches strong RAG pipelines in accuracy, and runs $3\times$--$19\times$ faster than RAG at inference.
\end{itemize}

\section{Related Work}

\paragraph{Retrieval-Augmented Generation}
RAG combines a parametric LLM with an external retriever to ground generation in retrieved text~\cite{lewis2020retrieval,izacard2020leveraging}, but introduces latency and engineering overhead at serving time. Generative retrieval~\cite{tay2022transformer,li2025matching,bevilacqua2022autoregressive} replaces vector search with a generative process predicting document IDs, yet still depends on an external corpus to resolve IDs into content. RING instead parameterizes retrieval entirely within the LLM, removing all runtime dependence on external search.

\vspace{-0.5em}
\paragraph{Parameterized Knowledge Injection}
Parameterized knowledge injection internalizes external information into model weights. Continual learning~\cite{diao2023survey} suffers from catastrophic forgetting on unstructured updates; knowledge editing (e.g., MEMIT~\cite{meng2022memit}) is limited to atomic facts; modular injection (KBLAM~\cite{wang2025kblam}, SR-KI~\cite{yu2025sr}) has memory linear in corpus size; Parametric RAG~\cite{su2025parametric} parameterizes documents via adapters but still queries an external retriever at inference; Temp-LoRA~\cite{wang2024greater} produces ephemeral adapters for long context, not permanent memorization. RING unifies memory, retrieval, and generation inside the LLM, achieving fully parametric, search-free augmentation at scale.

\section{RING}
\subsection{Model Architecture}\label{model_arch}
\paragraph{Overview}
As shown in Figure~\ref{model-arch}, RING replaces a dense decoder with a sparse MoE layer stack composed of a \emph{basic expert} $E_{\text{base}}$ and a \emph{knowledge expert} $E_{\text{knw}}$. This is initialized by duplicating the MLP weights; the \emph{basic expert} retains the original weights, while the \emph{knowledge expert} is initialized by the original MLP and used for subsequent new knowledge memorization. A \emph{Router} $r_\phi$ produces a sparse gate $g_t \in {0,1}$ at each layer and position, deciding which expert computes the token’s transformation. $E_{\text{base}}$ preserves general abilities and minimizes forgetting, while $E_{\text{knw}}$ parameterizes the external corpus and acts as an implicit index.

\paragraph{Sparse Routing}
At layer $\ell$ and position $t$ with hidden state $h^{(\ell)}_t$, the \emph{Router} first produces expert probabilities $p^{(\ell)}_t$ and makes a top\text{-}1 selection $g^{(\ell)}_t$:
\begin{align}
p^{(\ell)}_t &= \operatorname{softmax}\!\big(W_r^{(\ell)} h^{(\ell)}_t\big), \quad 
g^{(\ell)}_t = \arg\max_{k \in \mathcal{K}} \, p^{(\ell)}_{t,k}.\notag
\end{align}
Let the expert set be $\mathcal{K}=\{\text{base},\text{knw}\}$, and denote by $E_k^{(\ell)}(\cdot)$ the transformation of expert $k$ at layer $\ell$. With hard (top\text{-}1) routing (we omit LayerNorm and residual connections to simplify the expression):
\begin{equation}
\tilde h^{(\ell+1)}_t
= \sum_{k\in \mathcal{K}} \mathbf{1}\!\left[g^{(\ell)}_t = k\right]\,
E^{(\ell)}_{k}\!\left(h^{(\ell)}_t\right).
\end{equation}

\begin{figure}[t]
\begin{center}
\centerline{\includegraphics[width=\columnwidth]{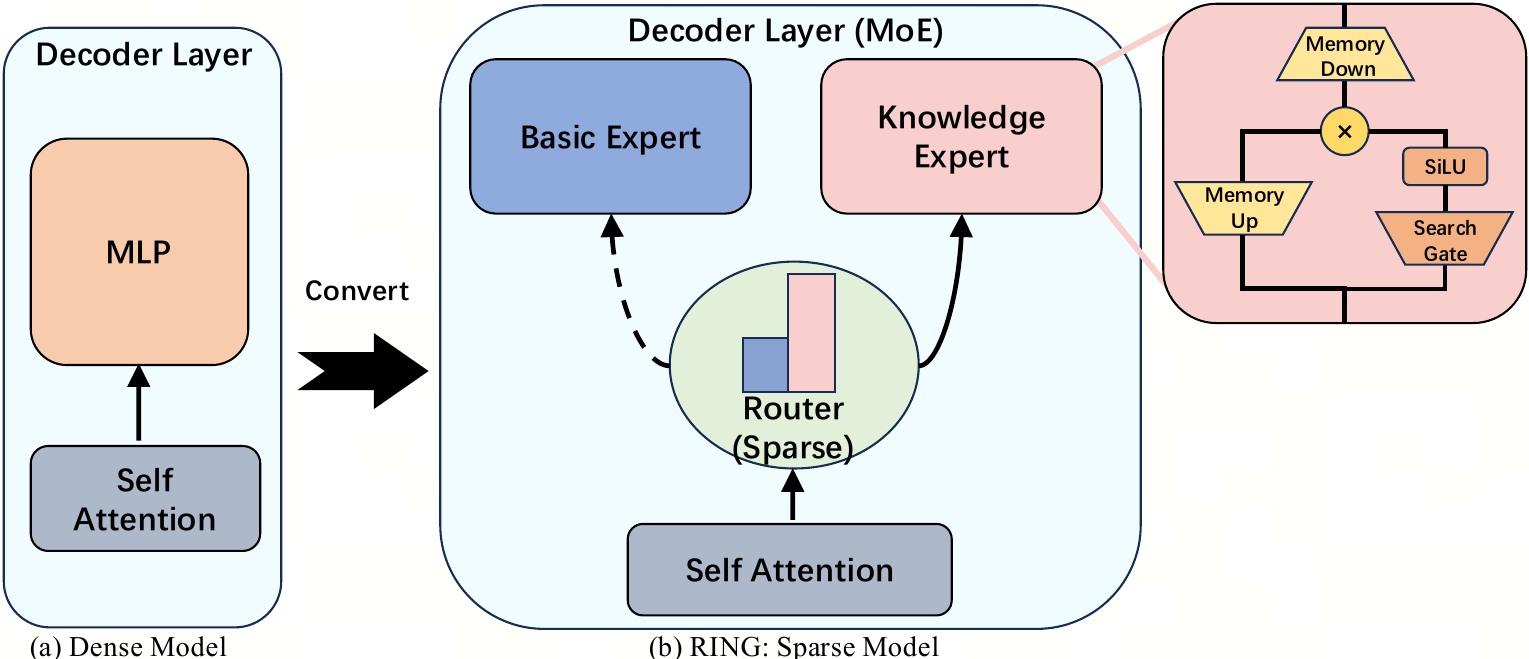}}
\caption{Model Architecture in our RING.}
\label{model-arch}
\end{center}
\end{figure}

\paragraph{Parameter Separation}
We further separate the parameters of \emph{knowledge expert} $E_{\text{knw}}$ into memory $E_{\text{knw}}^{m}$ and search $E_{\text{knw}}^{s}$. This ensures that subsequent training of search capabilities will not affect the original memory and allows for incremental addition of new memory data. As shown in Figure~\ref{model-arch}, $E_{\text{knw}}^{m}$ consists of memory up and down projection layers and $E_{\text{knw}}^{s}$ is a gate projection followed by a Sigmoid Linear Unit (SiLU) function. This design is inspired by existing observations that up and down projection layers in Transformers are important key-value memories responsible for organizing the massive amounts of knowledge in LLMs~\cite{geva2021transformer}. In our design, \emph{Memory Down} is used to store knowledge, while \emph{Memory Up} is a parametric index of the stored knowledge. \emph{Search Gate} works in conjunction with SiLU to reweight indexes for different queries in order to acquire corresponding knowledge in text generation. This parameter separation avoids interference between different training stages in RING.



\subsection{Continued Pre-training for Memorization} \label{CPT}
\paragraph{Data Preprocess}
Given a large-scale unstructured knowledge base $\mathcal{D}$ that requires LLMs to memorize (just like the corpus for retrieval in conventional RAG).
We perform data preprocessing to construct training data for CPT in RING. During this process, we segment each unstructured long document into multiple non-overlapping fragments, each consisting of three consecutive sentences. We combine the document title and the fragment content into a JSON-formatted text as:
\begin{flushleft}
\begingroup\ttfamily\small
\begin{verbatim}
u={"title": "title of the document",
  "content": "content of the fragment"}
\end{verbatim}
\endgroup
\end{flushleft}
which is the unit for memorization and search in the CPT, SFT, and RL stages. To ensure the semantic integrity of the long documents, we also add the full text of the document to the training set in CPT. However, this full-text data is only for memory assistance and will not be generated during the downstream search process. After the above process on the entire corpus, we obtain the CPT training set $\mathcal{D}_{\mathrm{CPT}}$.

\begin{figure}[t]
\begin{center}
\centerline{\includegraphics[width=\columnwidth]{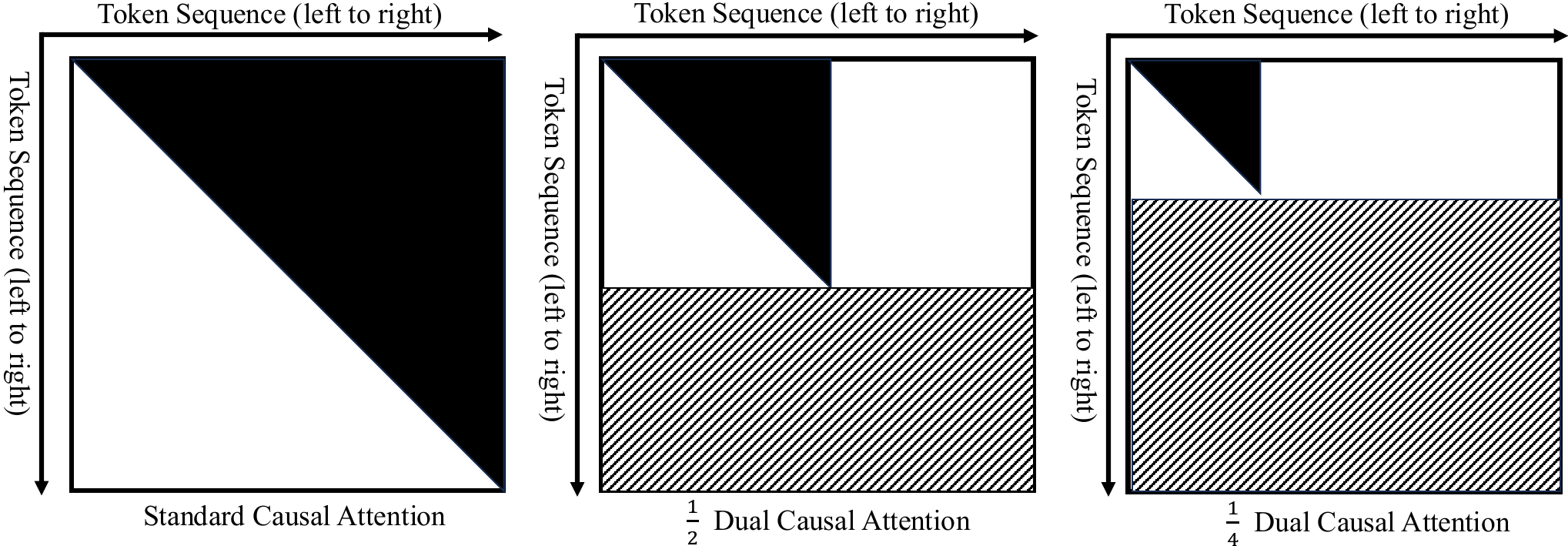}}
\caption{Dual Causal Attention. The white areas are visible tokens, the black areas are masked tokens, and the shaded areas represent invalid tokens.}
\label{dca}
\end{center}
\end{figure}

\paragraph{Dual Causal Attention}
CPT injects new knowledge into the Knowledge Expert. While bidirectional attention is known to be more effective for memorization~\cite{zhang2025bidirectional}, the LLM's autoregressive causal mask conflicts with it. We propose \textbf{Dual Causal Attention (DCA)}: each CPT example is run with three attention masks $M^{(\alpha)}$, $\alpha\in\{1,\tfrac12,\tfrac14\}$, that partition the sequence into a target block $B_\alpha$ (causally masked internally) and a complement $C_\alpha$ (visible as backward evidence); the resulting per-mask NLL losses are combined as $\mathcal{L}_{\mathrm{DCA}}=\sum_{\alpha}\lambda_{\alpha}\mathcal{L}^{(\alpha)}$ (Figure~\ref{dca}). $\alpha=1$ recovers standard causal pre-training; $\alpha\in\{\tfrac12,\tfrac14\}$ introduce bidirectional supervision \emph{without} leaking future targets inside $B_\alpha$. At inference, only $M^{(1)}$ is used, so generation remains strictly autoregressive. Formal mask definitions, the per-block NLL, and the combined loss are given in Appendix~\ref{sec:dca-formal}.

\paragraph{Training} CPT freezes the entire model except the Knowledge Expert's \emph{Memory Down} and \emph{Memory Up} projections, and forces the Router to dispatch every token to the Knowledge Expert. Optimizing $\mathcal{L}_{\mathrm{DCA}}$ on $\mathcal{D}_{\mathrm{CPT}}$ concentrates injected knowledge into the storing (Memory Down) and indexing (Memory Up) projections while preserving the base model's capabilities.

\subsection{SFT for Search-then-Answer}\label{sft}
\paragraph{Search-then-Answer Paradigm}
In the SFT stage, we introduce a novel \emph{Search-then-Answer} paradigm designed to integrate knowledge retrieval seamlessly within the LLM's generation process. When the LLM receives a query, it first identifies whether new, externally injected knowledge is required to answer the question. If so, it performs a parametric retrieval from the knowledge expert $E_{\text{knw}}$ to generate the top-1 relevant fragment. This fragment is enclosed within special tokens \texttt{<retrieval>} and \texttt{</retrieval>}. The retrieved fragment is typically short (with about 200 tokens) to avoid excessive additional computational overhead during generation. Once the relevant fragment is identified, the LLM transitions to the next phase where it generates the final answer within the \texttt{<answer>} and \texttt{</answer>} tags. This design follows the natural generative flow.

\paragraph{Instruction Fine-tuning Data and Loss}
To bootstrap the \emph{Search-then-Answer} behavior, we use \texttt{Qwen3-235B} to synthesize $\{(q_i, a_i, r_i)\}$ triples from documents in $\mathcal{D}_{\mathrm{CPT}}$, where $r_i$ is a reference span justifying $a_i$, and align each $r_i$ to a memory unit $u_i^\star\in\mathcal{D}_{\mathrm{CPT}}$ via exact string match (triples without a match are discarded). Each accepted triple is formatted as a search-then-answer target string with $u_i^\star$ in the \texttt{<retrieval>} block and $a_i$ in the \texttt{<answer>} block. We mix in no-retrieval instructions from a generic corpus to teach ``retrieve only when needed.'' The model is trained with standard cross-entropy on the target string. During SFT we freeze the CPT \emph{Memory Down/Up} and only update the Router and Search Gate, ensuring retrieval behavior is learned without overwriting memorized content. Full data recipe, target template, and loss formulation are in Appendix~\ref{datasets}.



\subsection{Reinforcement Learning for Search from the Parametric Knowledge}
RL stage in our RING aims to enable the LLM to learn a generalizable ability to retrieve information from its parametric knowledge base. Drawing on existing research that RL enhances the LLM's ability to sample more effectively from its parameter space~\cite{yue2025does,deng2026latent}, we find this process closely mirrors traditional retrieval, where the goal is to identify the most relevant document to answer a query from a library. Hence, we believe RL holds great potential in enabling the LLM to perform generative retrieval directly from its parametric knowledge base. After the SFT phase, which provides the LLM with a basic \emph{Search-then-Answer} paradigm, we apply RL to teach the LLM how to search for the most relevant memory units from the parametric knowledge base. The RL algorithm we use is GSPO~\cite{zheng2025group}. The reward structure is designed to guide the LLM's generative retrieval, rank and answer process in three ways:

\textbf{1. Format Reward.} ($R_{\text{format}}$) When facing instruction that is relevant to the external knowledge, the LLM must adhere to the \emph{Search-then-Answer} format. The response should first include the retrieval block enclosed within the tokens \texttt{<retrieval>} and \texttt{</retrieval>}, followed by the answer within \texttt{<answer>} and \texttt{</answer>}. If the format is correct, the reward is 1; otherwise, it is 0.

\textbf{2. Search Reward.} ($R_{\text{search}}$) This reward aims to guide LLM to find the most matching units from the external corpus $\mathcal{D}_{\mathrm{CPT}}$ memorized in the CPT stage, and output the memory units in \texttt{<retrieval>} ... \texttt{</retrieval>}, we introduce a dense reward mechanism to achieve this. This mechanism extracts the \texttt{title} and \texttt{content} of the memory unit within \texttt{<retrieval>} and \texttt{</retrieval>} of the LLM's generated sequences, and then calculates a series of rewards to guide the LLM's learning of both retrieval and re-ranking in its generation. The design principle of these rewards is to encourage LLM to generate units that are highly consistent with correct memory units. (1) {Learning to Retrieve Reward} ($R^{1}_{\text{dense}}$): This reward measures the normalized longest common substring (LCS) length between the retrieved memory unit's \texttt{title} and the \texttt{title} of the correct memory unit. A higher LCS indicates better retrieval performance, and a higher reward is given. (2) {Learning to Rerank Reward} ($R^{2}_{\text{dense}}$): This reward measures the normalized LCS between the retrieved memory unit's \texttt{content} and the \texttt{content} of the correct memory unit. This reward encourages the model to select more relevant content conditioned on the already fuzzy-matched titles.
(3) {Exact Match Reward} ($R^{3}_{\text{dense}}$): This reward checks if the reference passage (ref) is exactly present in the generated \texttt{content}. If the reference is present, the reward is 1; otherwise, the reward is 0. This reward helps the model to retrieve exact top-1 ground truth. The combined search reward is:
\begin{align}
    R_{\text{search}}=\lambda_{\text{a}}R^{1}_{\text{dense}} + \lambda_{\text{b}}R^{2}_{\text{dense}} + \lambda_{\text{c}}R^{3}_{\text{dense}},
\end{align}
 \(\lambda_{\text{a}}\), \(\lambda_{\text{b}}\), and \(\lambda_{\text{c}}\) are hyperparameters in Section~\ref{imple}.
 
We use exact substring matching rather than semantic matching because we observe with semantic matching, the model might learn to generate semantically similar, but incorrect, responses that still satisfy the reward conditions. Literal matching is straightforward and prevents such issues.

\textbf{3. Answer Reward.} ($R_{\text{answer}}$) The LLM's answer, enclosed in the \texttt{<answer>} block, is compared to the correct reference answer. We use \texttt{Qwen3-30B-A3B} to judge whether the LLM’s generated answer matches the correct answer. If the answer is correct, the reward is 1; if it is incorrect, the reward is 0. 

The total reward is the sum of the rewards:
\[
R_{\text{total}} = \lambda_{\text{format}} R_{\text{format}} + \lambda_{\text{answer}} R_{\text{answer}} + \lambda_{\text{search}} R_{\text{search}},
\]
where \(\lambda_{\text{format}}\), \(\lambda_{\text{answer}}\), and \(\lambda_{\text{search}}\) are hyperparameters in Section~\ref{imple}. By RL, LLM learns to retrieve and generate from the parametric knowledge base.


\section{Theoretical Analysis}
\label{sec:theory_main}

We formulate RING as a principled discrete latent variable model that approximates the classical RAG objective. This theoretically justifies our MoE architecture and the training pipeline. Detailed proofs are provided in Appendix~\ref{sec:theory}.

\paragraph{Latent Evidence and RING Approximation}
Let $x$ be the query, $y$ the answer, and $\mathcal{U}=\{u_z\}_{z=1}^N$ the universe of injected memory units. Conventional RAG marginalizes over a retrieved latent variable $z$:
\begin{equation}
    p(y\mid x) = \sum_{z=1}^N p(z\mid x, \mathcal{U}) \, p_\theta(y\mid x, u_z).
\end{equation}
Conventional RAG relies on a non-parametric $p(z\mid x, \mathcal{U})$ via external vector search. RING parameterizes this retrieval fully internally as:
\begin{equation}
    p_\Theta(y, z \mid x) = \underbrace{p_\phi(z \mid x)}_{\text{Router \& Gate}} \cdot \underbrace{p_{\theta_{\text{gen}}}(y \mid x, u_z; \Theta_{\text{mem}})}_{\text{Knowledge Expert}},\notag
\end{equation}
where $\Theta_{\text{mem}}$ stores $\mathcal{U}$. We prove in {Theorem~\ref{thm:app_containment}} that RING's hypothesis space contains explicit top-$K$ RAG, implying that parametric retrieval can theoretically match external search expressivity (Corollary~\ref{cor:app_risk}).

\paragraph{Knowledge Expert as Key-Value Memory}
We justify our specific Knowledge Expert architecture by showing it is structurally isomorphic to a differentiable Key-Value retrieval mechanism.
As derived in {Proposition~\ref{prop:app_kv_readout}}, the computation of the Knowledge Expert $E_{\text{knw}}(h)$ is:
\begin{align}
    E_{\text{knw}}(h) 
    &= \sum_{i} \underbrace{\text{SiLU}(\langle g_i, h \rangle) \langle k_i, h \rangle}_{\text{Attention Score } \alpha_i(h)} \, v_i,
\end{align}
where $k_i$ (rows of $W_{\text{up}}$) act as {Keys}, $v_i$ (cols of $W_{\text{down}}$) as {Values}, and $g_i$ as a query-dependent {Search Gate}.
{Lemma~\ref{lem:app_mips_softmax}} suggests that with learned gating, this mechanism approaches Maximum Inner Product Search, allowing precise retrieval from millions of parametric slots.

\paragraph{Training as Variational Inference}
We aim to maximize the marginal log-likelihood of generating the correct answer, $\log p_\Theta(y \mid x)$. Since marginalizing over all latent memory units $z$ is intractable, we maximize the Evidence Lower Bound (ELBO). Introducing a variational posterior $q(z \mid x, y)$ over the evidence (derivation in Eq.~\eqref{eq:app_elbo}):
\begin{align}
    \log p_\Theta(y \mid x) &\ge \mathbb{E}_{q(z \mid x, y)} \left[ \log \frac{p_\Theta(y, z \mid x)}{q(z \mid x, y)} \right] \nonumber \\
    &= \underbrace{\mathbb{E}_{q(z \mid x, y)}[\log p_{\theta_{\text{gen}}}(y \mid x, u_z)]}_{\text{Generation (Reconstruction)}} \notag \\ & -\underbrace{\text{KL}(q(z \mid x, y) \,||\, p_\phi(z \mid x))}_{\text{Retrieval (Alignment)}}. \notag
\end{align}
This decomposition provides a unified theoretical view of our three-stage pipeline:

CPT (Evidence Maximization): The generation term $\log p(y|x, u_z)$ fundamentally relies on the model's ability to comprehend and reconstruct the memory unit $u_z$. By minimizing the DCA loss (Prop.~\ref{prop:app_dca_monotone}), CPT injects $\mathcal{U}$ into $\Theta_{\text{mem}}$, ensuring that conditioned on a retrieved $z$, the likelihood of generating relevant tokens is maximized. We prove in Prop.~\ref{prop:app_dca_monotone} that DCA provides a tighter bound on conditional entropy than standard causal attention, theoretically guaranteeing more efficient storage.
    
SFT (Posterior Alignment): In SFT, we are given the ground truth evidence $u^\star$. This is equivalent to setting the variational posterior to a delta function $q(z \mid x, y) = \delta(z = z^\star)$. The ELBO maximization then simplifies to standard supervised learning:
    \begin{equation}
        \max \big( \underbrace{\log p_{\theta_{\text{gen}}}(y \mid x, u^\star)}_{\text{Answer Generation}} + \underbrace{\log p_\phi(z^\star \mid x)}_{\text{Retrieval Supervision}} \big).
    \end{equation}
    This forces the router policy $p_\phi$ to align with the "ideal" posterior defined by the dataset.
    
RL (Generalizing the Policy): When ground truth evidence is unavailable or multiple valid snippets exist, the fixed posterior $\delta(z=z^\star)$ is suboptimal. RL optimizes the policy $p_\phi(z|x)$ to maximize the expected reward. As shown in Lemma~\ref{lem:app_dense_credit}, our dense search rewards (e.g., LCS) serve as a proxy for the unobserved latent posterior, guiding the retrieval policy to maximize the ELBO's expectation term even without explicit supervision.






\section{Experiments}
\subsection{Datasets}\label{sec:datasets}
\textbf{News-2025} is our constructed bilingual benchmark of 170k Chinese news documents (341.93M tokens) from June--July 2025, translated to English with \texttt{DeepSeek-R1-0528}. Since the base LLM (Qwen3) predates this period, correct answers are hard to be obtained from pre-existing parametric knowledge or associative reasoning, isolating each method's true injection-and-utilization capability. We construct News-2025 to probe the two capabilities RING targets: {the injection of large-scale new knowledge} and {the utilization of this injected parametric knowledge}. The corpus is partitioned into subsets for CPT, SFT, and RL (Appendix~\ref{datasets}). The CPT set includes the whole documents collection, which serves as the external knowledge base $\mathcal{D}_{\mathrm{CPT}}$ that needs to be injected. The SFT set is designed to teach the model the “Search-then-Answer” paradigm. The RL set is used for the training of ``search from parametric knowledge''. As for {the test set,} we sample a held-out document set and use \texttt{GPT-5} to construct 10k QA pairs whose source documents do \emph{not} appear in SFT or RL, jointly probing the injection-and-utilization capability.

\begin{table*}[t]
\centering
\caption{Performance on News-2025 (English \& Chinese), reported as Accuracy (\%) and Time To First Token (TTFT, seconds). ``Top-$k$'' denotes the number of documents prepended to the LLM. ``+Rerank'' denotes retrieve-top-10 with the best embedding followed by a cross-encoder reranker selecting top-3. \textbf{Bold} marks the global best in each column; \underline{underline} marks the best within the \textsc{Knowledge Injection Without External Search} group (where RING belongs). When the same entry is best both globally and within the group, only bold is used.}
\label{tab:main_results_combined}
\resizebox{0.99\textwidth}{!}{
    \setlength{\tabcolsep}{4pt}
    \renewcommand{\arraystretch}{1.00}
    \begin{tabular}{ll cccc cccc}
    \toprule
    \multirow{3}{*}{\textbf{Method}} & \multirow{3}{*}{\textbf{Search}} & \multicolumn{4}{c}{\textbf{Based on Qwen3-8B}} & \multicolumn{4}{c}{\textbf{Based on Qwen3-14B}} \\
    \cmidrule(lr){3-6} \cmidrule(lr){7-10}
    & & \multicolumn{2}{c}{\textbf{English}} & \multicolumn{2}{c}{\textbf{Chinese}} & \multicolumn{2}{c}{\textbf{English}} & \multicolumn{2}{c}{\textbf{Chinese}} \\
    \cmidrule(lr){3-4} \cmidrule(lr){5-6} \cmidrule(lr){7-8} \cmidrule(lr){9-10}
    & & \textbf{Acc ($\uparrow$)} & \textbf{TTFT ($\downarrow$)} & \textbf{Acc ($\uparrow$)} & \textbf{TTFT ($\downarrow$)} & \textbf{Acc ($\uparrow$)} & \textbf{TTFT ($\downarrow$)} & \textbf{Acc ($\uparrow$)} & \textbf{TTFT ($\downarrow$)} \\
    \midrule

    Original LLM & None & 6.75 & 0.30 & 7.15 & 0.35 & 7.26 & 0.57 & 7.73 & 0.48 \\
    \midrule
    \multicolumn{10}{l}{\textsc{Knowledge Injection With External Search (RAG, Top-1)}} \\
    & BGE-M3 & 28.57 & 1.32 & 28.94 & 1.41 & 26.54 & 2.20 & 27.88 & 2.32 \\
    & Gemini-Emb-Exp & 32.76 & 3.46 & 34.12 & 3.52 & 30.72 & 4.25 & 32.17 & 4.39 \\
    RAG & OpenAI-Emb-3L & 29.99 & 3.95 & 30.97 & 3.88 & 27.96 & 4.80 & 29.35 & 4.54 \\
    & Qwen3-0.6B-Emb & 30.84 & 1.29 & 32.17 & 1.45 & 30.41 & 2.21 & 31.84 & 2.19 \\
    & Qwen3-8B-Emb & 33.83 & 1.81 & 35.28 & 1.73 & 31.80 & 2.60 & 33.26 & 2.55 \\
    \midrule
    \multicolumn{10}{l}{\textsc{Stronger RAG Pipelines}} \\
    RAG (Top-3) & Qwen3-8B-Emb & 36.42 & 2.85 & 37.91 & 2.78 & 34.05 & 3.55 & 35.62 & 3.42 \\
    RAG (Top-5) & Qwen3-8B-Emb & 37.18 & 4.05 & 38.65 & 3.92 & 35.21 & 4.85 & 36.84 & 4.68 \\
    RAG (Top-10) & Qwen3-8B-Emb & 36.85 & 6.45 & 38.34 & 6.22 & 35.94 & 7.62 & 37.61 & 7.35 \\
    \,\,+Rerank (BGE-v2-m3) & Qwen3-8B-Emb & 38.73 & 3.45 & 40.21 & 3.32 & 37.42 & 4.18 & 39.18 & 4.05 \\
    \,\,+Rerank (Qwen3-Rerank) & Qwen3-8B-Emb & \textbf{39.51} & 4.05 & \textbf{41.05} & 3.88 & 38.26 & 4.85 & \textbf{39.82} & 4.68 \\
    \,\,+HyDE & Qwen3-8B-Emb & 36.94 & 5.15 & 38.42 & 4.92 & 34.85 & 6.75 & 36.51 & 6.48 \\
    \,\,+Query Rewrite & Qwen3-8B-Emb & 35.62 & 3.85 & 37.18 & 3.68 & 33.74 & 4.72 & 35.42 & 4.55 \\
    \midrule
    \multicolumn{10}{l}{\textsc{Knowledge Injection Without External Search}} \\
    Lora-Tuning & None & 21.37 & \textbf{0.34} & 22.12 & \textbf{0.32} & 22.70 & 0.65 & 26.37 & 0.60 \\
    Full-Tuning & None & 26.92 & 0.39 & 27.06 & 0.34 & 28.78 & \textbf{0.60} & 31.83 & \textbf{0.58} \\
    KBLAM & None & 29.13 & 2.25 & 28.32 & 2.15 & 33.21 & 3.65 & 33.48 & 2.98 \\
    SR-KI & None & 31.85 & 2.03 & 29.22 & 2.09 & 36.65 & 3.20 & 33.50 & 2.85 \\
    LAG & None & 24.51 & 0.78 & 23.86 & 0.82 & 28.34 & 1.18 & 27.62 & 1.12 \\
    MLP Memory & None & 28.45 & 0.55 & 27.34 & 0.52 & 32.18 & 0.88 & 30.95 & 0.85 \\
    AtlasKV & None & 32.46 & 1.85 & 30.18 & 1.72 & 37.62 & 2.65 & 34.85 & 2.48 \\
    \rowcolor{gray!15} \textbf{RING (ours)} & None & \underline{35.08} & 0.40 & \underline{32.36} & 0.49 & \textbf{41.19} & 0.78 & \underline{37.92} & 0.75 \\
    \bottomrule
    \end{tabular}
}
\end{table*}
\subsection{Baselines}
We group baselines into three categories (Table~\ref{tab:main_results_combined}; full descriptions in Appendix~\ref{baselines-details}):
\textbf{(i) Conventional RAG (Top-1)} with SOTA embedders BGE-M3~\cite{chen2024bge}, OpenAI-Emb-3L, Gemini-Embedding~\cite{lee2025gemini}, and Qwen3-Embedding~\cite{zhang2025qwen3};
\textbf{(ii) Stronger RAG Pipelines} (all using Qwen3-8B-Emb): Top-$k$ retrieval with $k\in\{3,5,10\}$, cross-encoder reranking on top-10 to top-3 (BGE-Reranker-v2-m3~\cite{li2023making} and Qwen3-Reranker-4B~\cite{zhang2025qwen3}), HyDE~\cite{gao2023precise}, and base-LLM Query Rewrite;
\textbf{(iii) Parametric Knowledge Injection}: Full Tuning, LoRA Tuning, modular methods KBLAM~\cite{wang2025kblam} and SR-KI~\cite{yu2025sr}, and three recent retriever-imitating methods closest to RING in spirit---LAG~\cite{fleshman2025lora} (per-document LoRA + data-free routing), MLP Memory~\cite{wei2025mlp} ($k$NN-distribution imitation), and AtlasKV~\cite{huang2025atlaskv} (compressed KV-cache injection). Unlike all parametric baselines whose retrieval mechanism is hand-designed and frozen, RING (i) keeps a separate Basic Expert to prevent forgetting, (ii) injects via bidirectional DCA, and (iii) learns retrieval end-to-end via RL with dense rewards.



\subsection{Main Results}
\label{sec:main_results}
The specific experimental settings can be found in Section~\ref{exp_setting}. The main results are in Table~\ref{tab:main_results_combined}.

\noindent\textbf{vs.\ Parametric Knowledge Injection.}
RING dominates every parametric injection baseline on both languages and both backbones. With Qwen3-8B, RING reaches 35.08\%/32.36\% (EN/ZH), exceeding the strongest prior method SR-KI by 3.23 and 3.14 points. Among the three recently proposed retriever-imitating baselines, AtlasKV is the closest competitor since KV-level injection retains finer-grained semantics than logit-level interpolation; however, its KG2KV pipeline loses information when news passages cannot be cleanly decomposed into $\langle h,r,t\rangle$ triples, and hierarchical KV pruning still incurs non-trivial attention overhead (TTFT 1.85s vs.\ 0.40s). MLP Memory is fast (TTFT 0.55s, comparable to Full-Tuning) but its probability-interpolation interface cannot resolve fine-grained factual queries with the same precision as KV injection. LAG is the weakest of the three: data-free per-token adapter routing was designed for libraries on the order of $10^3$ adapters, and false selections compound as the library scales to our 170k-document corpus. The advantage widens on Qwen3-14B, where RING exceeds every parametric baseline by at least 3.5 points.

\noindent\textbf{vs.\ External-Search RAG.}
RING occupies the accuracy--latency Pareto frontier. On the 8B backbone, the strongest pipeline (Qwen3-Reranker on top-10, output top-3) surpasses RING by 4.43 EN/8.69 ZH points but at $10\times$ TTFT (4.05s vs.\ 0.40s); on 14B the trade-off flips---RING (41.19/37.92) beats this same pipeline on EN and trails by under 2 points on ZH while remaining $6\times$ faster. Lighter Top-$k$ retrieval plateaus near 37\% on 8B with diminishing returns ($k{=}10$ pushes TTFT to 6.45s); HyDE and Query Rewrite add LLM forwards without statistically meaningful gains over Top-3, since the base LLM cannot expand queries for post-cutoff news.

\noindent\textbf{Latency.}
RING's TTFT matches Full-Tuning---the only other paradigm without inference-time computation beyond a forward pass---and is the lowest among all evaluated methods: $3$--$19\times$ faster than every RAG variant and $3$--$6\times$ faster than KV-injection baselines, confirming that RING internalizes retrieval at minimal serving cost.

\begin{figure}[t]
    \centering
    \begin{subfigure}{0.48\linewidth}
        \centering
        \includegraphics[width=\linewidth]{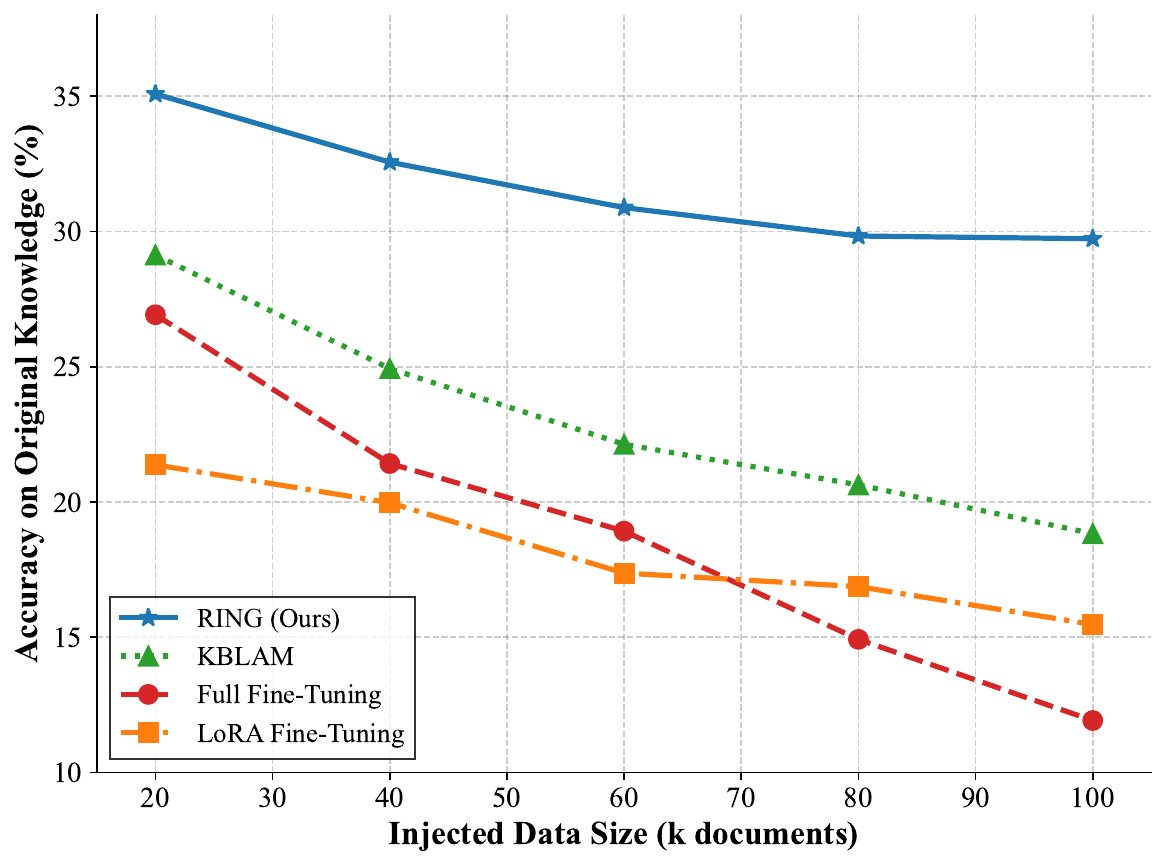}
        \caption{Sequential Injection}
        \label{fig:sequential}
    \end{subfigure}
    \hfill 
    \begin{subfigure}{0.48\linewidth}
        \centering
        \includegraphics[width=\linewidth]{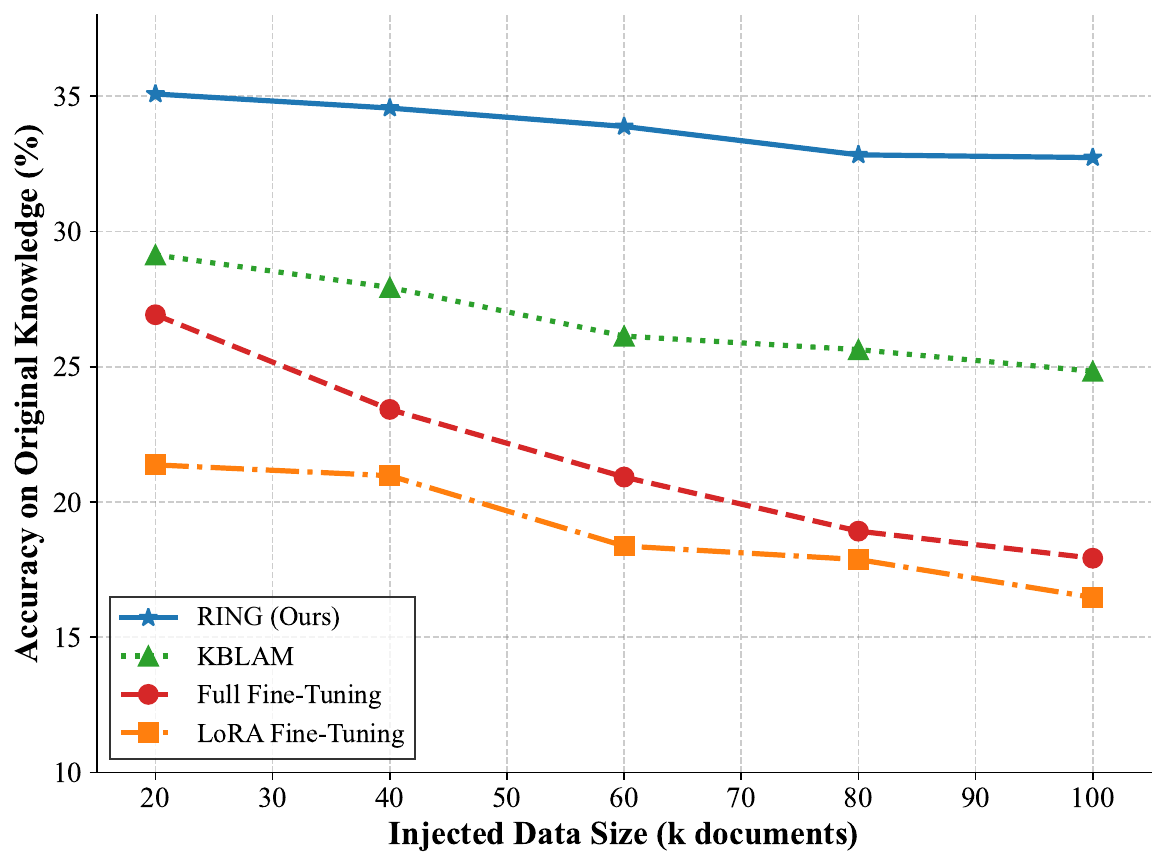}
        \caption{Mixed Training}
        \label{fig:mixed}
    \end{subfigure}
    
    \caption{Comparison of knowledge retention strategies. 
    \textbf{(a)} Sequential injection means the model is trained strictly on the sequence of new knowledge. 
    \textbf{(b)} Mixed training means introduces a
replay buffer of original data during training to mitigate forgetting.}
    \label{fig:forgetting_comparison}
\end{figure}
\subsection{Analysis}
\paragraph{Incremental Updating of New Knowledge}
We evaluate robustness to catastrophic forgetting by injecting 20k to 100k documents and comparing against baselines (Figure~\ref{fig:forgetting_comparison}) under two settings. \textbf{(1) Sequential Injection:} training only on new data causes severe forgetting for conventional methods: \textit{Full Fine-Tuning} drops from $\sim$27\% to $\sim$12\% at 100k, with \textit{LoRA} and \textit{KBLAM} also degrading noticeably. In contrast, RING remains stable, keeping accuracy above 30\%, indicating effective isolation of new memories via parameter separation. \textbf{(2) Mixed Training:} adding replay flattens the baselines’ curves, but RING still achieves the best overall retention, staying $>33\%$ while \textit{KBLAM} and \textit{Full Fine-Tuning} plateau around $\sim$25\% to $\sim$18\%.

\begin{table}[t]
\centering
\caption{Ablation study of RING. Performance is reported as the accuracy (\%) on English dataset.}
\label{tab:ablation}

\resizebox{\columnwidth}{!}{
    \setlength{\tabcolsep}{15pt} 
    \renewcommand{\arraystretch}{1.00}

    \newcommand{\baseblock}[1]{%
      \rowcolor{gray!25}%
      \multicolumn{2}{l}{\textbf{\textit{#1}}} \\}
\scalebox{1.00}{
    \begin{tabular}{lc}
    \toprule
    \textbf{Configuration} & \textbf{Accuracy ($\uparrow$)} \\
    \midrule

    \baseblock{1. Impact of Model Architecture}
 
    Dense (Standard) & 30.59 \\
    Sparse (MoE) & 32.71 \\
    Sparse + Gate (Ours) & \textbf{35.08} \\
    \midrule

    \baseblock{2. Impact of Attention Mechanism}

    Uni-directional Attention (Standard) & 33.49 \\
    Dual Causal Attention (Ours) & \textbf{35.08} \\
    \midrule
    
    \baseblock{3. Impact of Training Stages}

    SFT Only  & 28.41 \\
    SFT + RL  & \textbf{35.08} \\
    \midrule

    \baseblock{4. Impact of Reward Components in RL}

    RL w/ $R_{\text{format}}$ + $R_{\text{answer}}$ (Answer Reward) & 32.17 \\
    RL w/ $R_{\text{format}}$ + $R_{\text{answer}}$ +  $R_{\text{search}}$ (Search Reward) & \textbf{35.08} \\
    \bottomrule
    \end{tabular}
    }
}
\end{table}

\paragraph{Ablation Study}
We ablate key components of RING in Table~\ref{tab:ablation}. {(1) Architecture:} Sparse MoE improves 30.59\%$\rightarrow$32.71\%, and adding our gating further boosts to 35.08\%. {(2) Attention:} DCA outperforms standard causal attention (35.08\% vs. 33.49\%). {(3) Training:} RL is essential---SFT only reaches 28.41\%, while SFT+RL achieves 35.08\%. {(4) Rewards:} search reward beats answer-only reward (35.08\% vs. 32.17\%), showing the benefit of fine-grained retrieval supervision.

\begin{figure}[t]
    \centering
    \begin{subfigure}{0.48\linewidth}
        \centering
{\includegraphics[width=\columnwidth]{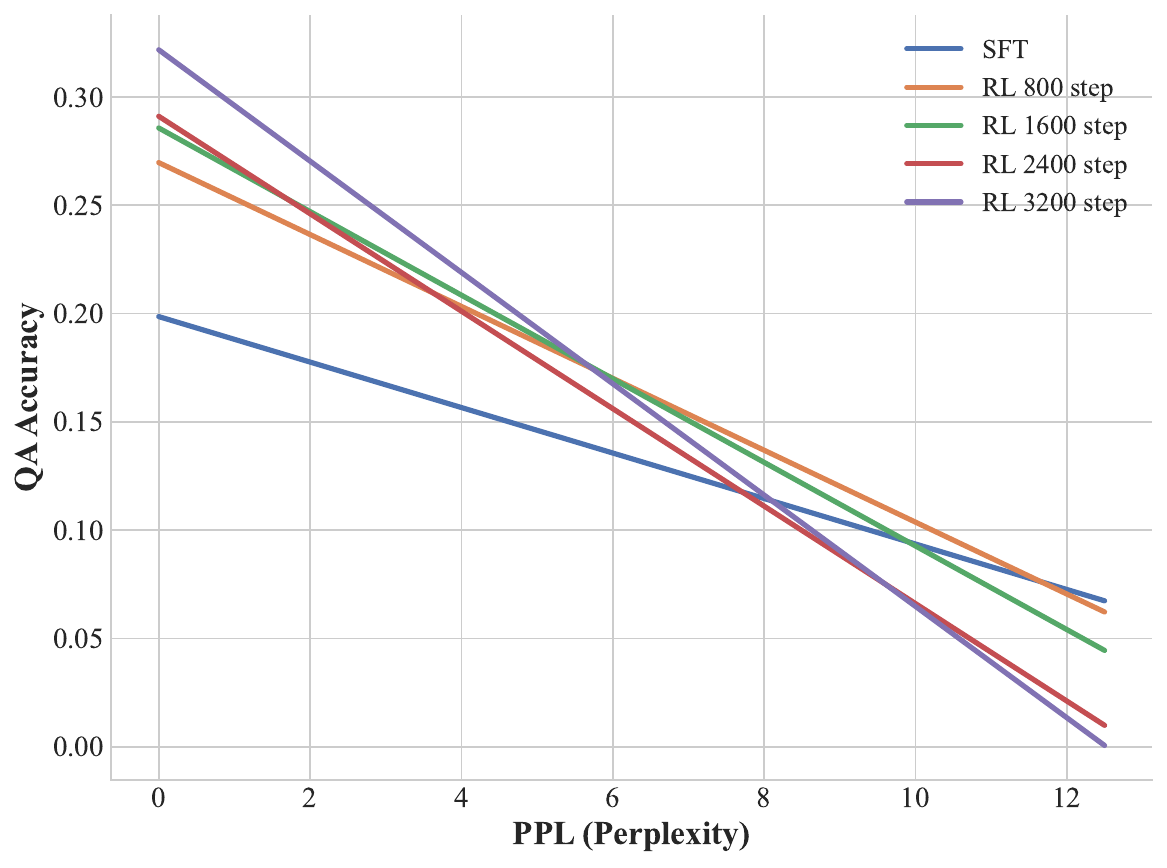}}
\caption{PPL-Accuracy mapping through RL training steps. }
        \label{fig:effect_of_rl_a}
    \end{subfigure}
    \hfill 
    \begin{subfigure}{0.48\linewidth}
        \centering
        \includegraphics[width=\linewidth]{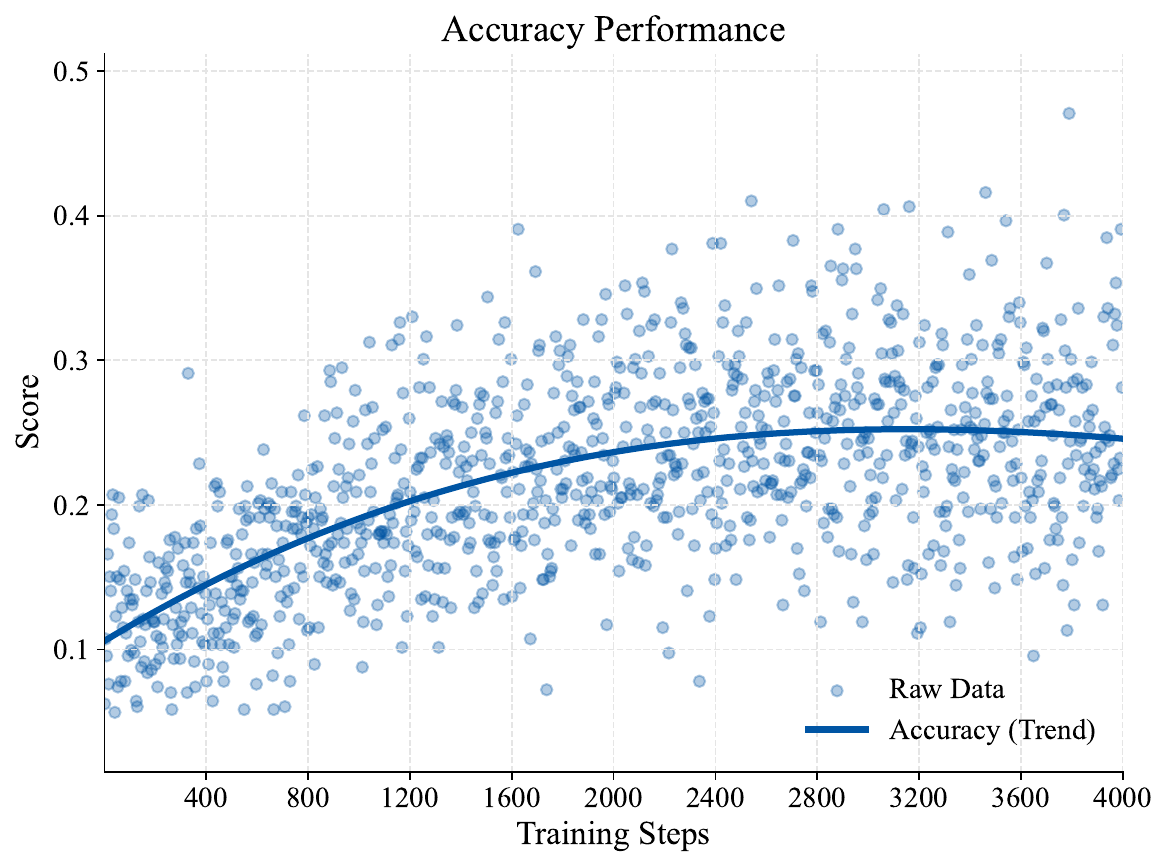}
        \caption{(Q\&A accuracy) varies by the RL training steps.}
        \label{fig:effect_of_rl_b}
    \end{subfigure}
    
    \caption{(a) We plot the linear relationship between Perplexity ($x$) and Q\&A Accuracy ($y$) for SFT and four RL checkpoints. (b) We show the training dynamics of Q\&A accuracy in RL training.}
    \label{fig:effect_of_rl}
\end{figure}

\paragraph{Effect of RL}
To investigate how RL facilitates the utilization of parametric knowledge, we analyze the relationship between memorization quality (PPL) and QA accuracy. Figure~\ref{fig:effect_of_rl_a} shows two key findings. {(1) Better utilization under good memorization:} in the low-PPL region, RL markedly improves accuracy (SFT $\sim$20\% $\rightarrow$ RL $>32\%$), indicating that memorization alone is not sufficient and RL helps the model reliably activate and use stored knowledge. {(2) Stronger grounding on parametric memory:} as RL proceeds, the PPL--accuracy curve becomes much steeper---accuracy becomes tightly coupled with memorization quality. When PPL is high, the RL model degrades sharply, suggesting reduced reliance on pre-trained priors and fewer ``lucky guesses.'' RL bridges memorization and effective retrieval-based usage of parametric knowledge.

\section{Conclusion}
We introduce RING, a fully parametric, search-free RAG framework that unifies indexing, retrieval, and generation inside a LLM. RING injects large-scale new knowledge into a dedicated Knowledge Expert via continued pre-training with Dual Causal Attention, and learns to \emph{search-then-answer} through supervised fine-tuning and reinforcement learning. We further provide a latent-variable view connecting RING to the classical RAG objective, showing it as a search-free approximation with a principled inference interpretation. Experiments show that RING matches or exceeds search-based RAG baselines while reducing inference latency.

\section*{Limitations}

First, RING targets settings where a large knowledge corpus is relatively stable and can be
internalized before deployment. Unlike conventional RAG systems, which can update an external
index by adding or deleting documents, RING requires additional training to refresh the
parametric memory. This makes it less suitable for highly dynamic knowledge sources unless
efficient incremental update strategies are developed.

Second, although RING generates an explicit \texttt{<retrieval>} block, this generated evidence is
not identical to returning a verbatim document from an external database. The model may compress,
paraphrase, or occasionally distort memorized content, which can make auditing and provenance
tracking harder than in search-based RAG. Applications that require strict citation fidelity or
legal traceability may still benefit from hybrid designs that combine parametric retrieval with
external verification.

Third, our experiments focus on knowledge-intensive question answering over a constructed news
corpus. While this setting stresses large-scale memorization and evidence selection, it does not
cover all retrieval scenarios, such as multi-hop web search, rapidly changing facts, multilingual
knowledge bases, or long-form generation with many cited sources. Extending RING to these settings
is an important direction for future work.

Finally, the proposed training pipeline introduces non-trivial upfront cost. Continued
pre-training with Dual Causal Attention, supervised search-then-answer learning, and RL with dense
retrieval rewards require carefully constructed data and substantial computation. The inference
benefits of search-free generation should therefore be weighed against the cost of building and
maintaining the parametric memory.

\section*{Ethical Considerations}
This paper proposes RING, a search-free retrieval-augmented generation framework that internalizes indexing and retrieval into model parameters to reduce latency and deployment complexity. The primary positive impact is enabling more efficient and accessible knowledge-augmented systems, especially in settings where maintaining external retrieval infrastructure is costly or impractical.

Potential risks are similar to those of other LLM-based systems. Because retrieval is performed parametrically, errors in memorization or retrieval may lead to confident but incorrect outputs, and harmful or biased content present in the injected corpus could be reinforced or surfaced during generation. In addition, internalizing knowledge may make it harder to audit, trace, or remove specific information compared to external retrieval, raising concerns around data governance and privacy if sensitive data were injected.

To mitigate these risks, RING should be used with careful dataset curation and filtering, privacy-preserving data handling, and evaluation for bias and safety. Where high-stakes decisions are involved, we recommend adding safeguards such as refusal policies, post-hoc verification, and provenance-aware logging or optional external validation. We hope this work encourages further research on controllable, auditable, and safe parametric knowledge integration.
\bibliography{main}

@article{zhang2024comprehensive,
  title={A comprehensive study of knowledge editing for large language models},
  author={Zhang, Ningyu and Yao, Yunzhi and Tian, Bozhong and Wang, Peng and Deng, Shumin and Wang, Mengru and Xi, Zekun and Mao, Shengyu and Zhang, Jintian and Ni, Yuansheng and others},
  journal={arXiv preprint arXiv:2401.01286},
  year={2024}
}

@article{lewis2020retrieval,
  title={Retrieval-augmented generation for knowledge-intensive nlp tasks},
  author={Lewis, Patrick and Perez, Ethan and Piktus, Aleksandra and Petroni, Fabio and Karpukhin, Vladimir and Goyal, Naman and K{\"u}ttler, Heinrich and Lewis, Mike and Yih, Wen-tau and Rockt{\"a}schel, Tim and others},
  journal={Advances in neural information processing systems},
  volume={33},
  pages={9459--9474},
  year={2020}
}

@inproceedings{izacard2020leveraging,
  title={Leveraging passage retrieval with generative models for open domain question answering},
  author={Izacard, Gautier and Grave, Edouard},
  booktitle={Proceedings of the 16th conference of the european chapter of the association for computational linguistics: main volume},
  pages={874--880},
  year={2021}
}

@article{li2025matching,
  title={From matching to generation: A survey on generative information retrieval},
  author={Li, Xiaoxi and Jin, Jiajie and Zhou, Yujia and Zhang, Yuyao and Zhang, Peitian and Zhu, Yutao and Dou, Zhicheng},
  journal={ACM Transactions on Information Systems},
  volume={43},
  number={3},
  pages={1--62},
  year={2025},
  publisher={ACM New York, NY}
}

@article{diao2023survey,
  title={Continual learning of large language models: A comprehensive survey},
  author={Shi, Haizhou and Xu, Zihao and Wang, Hengyi and Qin, Weiyi and Wang, Wenyuan and Wang, Yibin and Wang, Zifeng and Ebrahimi, Sayna and Wang, Hao},
  journal={ACM Computing Surveys},
  volume={58},
  number={5},
  pages={1--42},
  year={2025},
  publisher={ACM New York, NY}
}

@article{deng2026latent,
  title={Latent-GRPO: Group relative policy optimization for latent reasoning},
  author={Deng, Jingcheng and Wei, Zihao and Pang, Liang and Wu, Junhong and Xu, Shicheng and Duan, Zenghao and Shen, Huawei},
  journal={arXiv preprint arXiv:2604.27998},
  year={2026}
}

@article{deng2025latent,
  title={Latent reasoning in llms as a vocabulary-space superposition},
  author={Deng, Jingcheng and Pang, Liang and Wei, Zihao and Xu, Shicheng and Duan, Zenghao and Xu, Kun and Song, Yang and Shen, Huawei and Cheng, Xueqi},
  year={2025}
}

@article{wei2026dynamic,
  title={Dynamic Rollout Editing for Reducing Overthinking in RL-Trained Reasoning Models},
  author={Wei, Zihao and Shi, Wenjie and Pang, Liang and Deng, Jingcheng and Xu, Shicheng and Guo, Shasha and Duan, Zenghao and Liu, Jiahao and Wang, Jingang and Shen, Huawei and others},
  journal={arXiv preprint arXiv:2606.17890},
  year={2026}
}

@article{duan2026skillattack,
  title={Skillattack: Automated red teaming of agent skills through attack path refinement},
  author={Duan, Zenghao and Tian, Yuxin and Yin, Zhiyi and Pang, Liang and Deng, Jingcheng and Wei, Zihao and Xu, Shicheng and Ge, Yuyao and Cheng, Xueqi},
  journal={arXiv preprint arXiv:2604.04989},
  year={2026}
}

@article{meng2022memit,
  title={Mass-editing memory in a transformer},
  author={Meng, Kevin and Sharma, Arnab Sen and Andonian, Alex and Belinkov, Yonatan and Bau, David},
  journal={arXiv preprint arXiv:2210.07229},
  year={2022}
}

@inproceedings{wang2025kblam,
  title={Kblam: Knowledge base augmented language model},
  author={Wang, Xi and Isazawa, Taketomo and Mikaelyan, Liana and Hensman, James},
  booktitle={International Conference on Learning Representations},
  volume={2025},
  pages={51629--51658},
  year={2025}
}

@article{yue2025does,
  title={Does reinforcement learning really incentivize reasoning capacity in llms beyond the base model?},
  author={Chen, Zhiqi and Lu, Rui and Zhao, Andrew and Wang, Zhaokai and Yue, Yang and Song, Shiji and Huang, Gao},
  journal={Advances in Neural Information Processing Systems},
  volume={38},
  pages={57654--57689},
  year={2026}
}

@article{nguyen2019toward,
  title={Toward understanding catastrophic forgetting in continual learning},
  author={Nguyen, Cuong V and Achille, Alessandro and Lam, Michael and Hassner, Tal and Mahadevan, Vijay and Soatto, Stefano},
  journal={arXiv preprint arXiv:1908.01091},
  year={2019}
}

@inproceedings{deng2024everything,
  title={Everything is editable: Extend knowledge editing to unstructured data in large language models},
  author={Deng, Jingcheng and Wei, Zihao and Pang, Liang and Ding, Hanxing and Shen, Huawei and Cheng, Xueqi},
  booktitle={International Conference on Learning Representations},
  volume={2025},
  pages={628--652},
  year={2025}
}

@article{wang2024comprehensive,
  title={A comprehensive survey of continual learning: Theory, method and application},
  author={Wang, Liyuan and Zhang, Xingxing and Su, Hang and Zhu, Jun},
  journal={IEEE transactions on pattern analysis and machine intelligence},
  volume={46},
  number={8},
  pages={5362--5383},
  year={2024},
  publisher={IEEE}
}

@inproceedings{yu2025sr,
  title={SR-KI: Scalable and Real-Time Knowledge Integration into LLMs via Supervised Attention},
  author={Yu, Bohan and Huang, Wei and Liu, Kang},
  booktitle={Proceedings of the AAAI Conference on Artificial Intelligence},
  volume={40},
  number={41},
  pages={34486--34494},
  year={2026}
}

@article{zheng2025group,
  title={Group sequence policy optimization},
  author={Zheng, Chujie and Liu, Shixuan and Li, Mingze and Chen, Xiong-Hui and Yu, Bowen and Gao, Chang and Dang, Kai and Liu, Yuqiong and Men, Rui and Yang, An and others},
  journal={arXiv preprint arXiv:2507.18071},
  year={2025}
}

@inproceedings{hu2024towards,
  title={Towards understanding factual knowledge of large language models},
  author={Hu, Xuming and Chen, Junzhe and Li, Xiaochuan and Guo, Yufei and Wen, Lijie and Yu, Philip and Guo, Zhijiang},
  booktitle={International Conference on Learning Representations},
  volume={2024},
  pages={28680--28715},
  year={2024}
}

@misc{li2023making,
      title={Making Large Language Models A Better Foundation For Dense Retrieval}, 
      author={Chaofan Li and Zheng Liu and Shitao Xiao and Yingxia Shao},
      year={2023},
      eprint={2312.15503},
      archivePrefix={arXiv},
      primaryClass={cs.CL}
}

@inproceedings{xu2024unsupervised,
  title={Unsupervised information refinement training of large language models for retrieval-augmented generation},
  author={Xu, Shicheng and Pang, Liang and Yu, Mo and Meng, Fandong and Shen, Huawei and Cheng, Xueqi and Zhou, Jie},
  booktitle={Proceedings of the 62nd annual meeting of the association for computational linguistics (volume 1: Long papers)},
  pages={133--145},
  year={2024}
}

@inproceedings{xu2025theory,
  title={A theory for token-level harmonization in retrieval-augmented generation},
  author={Xu, Shicheng and Pang, Liang and Shen, Huawei and Cheng, Xueqi},
  booktitle={International Conference on Learning Representations},
  volume={2025},
  pages={36616--36642},
  year={2025}
}

@article{wang2023survey,
  title={Survey on factuality in large language models: Knowledge, retrieval and domain-specificity},
  author={Wang, Cunxiang and Liu, Xiaoze and Yue, Yuanhao and Tang, Xiangru and Zhang, Tianhang and Jiayang, Cheng and Yao, Yunzhi and Gao, Wenyang and Hu, Xuming and Qi, Zehan and others},
  journal={arXiv preprint arXiv:2310.07521},
  year={2023}
}

@inproceedings{xu2024search,
  title={Search-in-the-chain: Interactively enhancing large language models with search for knowledge-intensive tasks},
  author={Xu, Shicheng and Pang, Liang and Shen, Huawei and Cheng, Xueqi and Chua, Tat-Seng},
  booktitle={Proceedings of the ACM Web Conference 2024},
  pages={1362--1373},
  year={2024}
}

@inproceedings{su2025parametric,
  title={Parametric retrieval augmented generation},
  author={Su, Weihang and Tang, Yichen and Ai, Qingyao and Yan, Junxi and Wang, Changyue and Wang, Hongning and Ye, Ziyi and Zhou, Yujia and Liu, Yiqun},
  booktitle={Proceedings of the 48th International ACM SIGIR Conference on Research and Development in Information Retrieval},
  pages={1240--1250},
  year={2025}
}

@article{gao2023retrieval,
  title={Retrieval-augmented generation for large language models: A survey},
  author={Gao, Yunfan and Xiong, Yun and Gao, Xinyu and Jia, Kangxiang and Pan, Jinliu and Bi, Yuxi and Dai, Yixin and Sun, Jiawei and Wang, Haofen and Wang, Haofen and others},
  journal={arXiv preprint arXiv:2312.10997},
  volume={2},
  number={1},
  pages={32},
  year={2023}
}

@article{wang2024greater,
  title={With greater text comes greater necessity: Inference-time training helps long text generation},
  author={Wang, Yan and Ma, Dongyang and Cai, Deng},
  journal={arXiv preprint arXiv:2401.11504},
  year={2024}
}

@article{chen2024bge,
  title={Bge m3-embedding: Multi-lingual, multi-functionality, multi-granularity text embeddings through self-knowledge distillation},
  author={Chen, Jianlv and Xiao, Shitao and Zhang, Peitian and Luo, Kun and Lian, Defu and Liu, Zheng},
  journal={arXiv preprint arXiv:2402.03216},
  volume={4},
  number={5},
  year={2024}
}

@article{zhang2025bidirectional,
  title={Bidirectional lms are better knowledge memorizers? a benchmark for real-world knowledge injection},
  author={Zhang, Yuwei and Yu, Wenhao and Feng, Shangbin and Zhu, Yifan and Peng, Letian and Srinivasa, Jayanth and Liu, Gaowen and Shang, Jingbo},
  journal={arXiv preprint arXiv:2505.12306},
  year={2025}
}

@inproceedings{geva2021transformer,
  title={Transformer feed-forward layers are key-value memories},
  author={Geva, Mor and Schuster, Roei and Berant, Jonathan and Levy, Omer},
  booktitle={Proceedings of the 2021 Conference on Empirical Methods in Natural Language Processing},
  pages={5484--5495},
  year={2021}
}

@article{bevilacqua2022autoregressive,
  title={Autoregressive search engines: Generating substrings as document identifiers},
  author={Bevilacqua, Michele and Ottaviano, Giuseppe and Lewis, Patrick and Yih, Scott and Riedel, Sebastian and Petroni, Fabio},
  journal={Advances in Neural Information Processing Systems},
  volume={35},
  pages={31668--31683},
  year={2022}
}

@article{tay2022transformer,
  title={Transformer memory as a differentiable search index},
  author={Tay, Yi and Tran, Vinh and Dehghani, Mostafa and Ni, Jianmo and Bahri, Dara and Mehta, Harsh and Qin, Zhen and Hui, Kai and Zhao, Zhe and Gupta, Jai and others},
  journal={Advances in neural information processing systems},
  volume={35},
  pages={21831--21843},
  year={2022}
}

@article{zhang2025qwen3,
  title={Qwen3 embedding: Advancing text embedding and reranking through foundation models},
  author={Zhang, Yanzhao and Li, Mingxin and Long, Dingkun and Zhang, Xin and Lin, Huan and Yang, Baosong and Xie, Pengjun and Yang, An and Liu, Dayiheng and Lin, Junyang and others},
  journal={arXiv preprint arXiv:2506.05176},
  year={2025}
}

@inproceedings{gao2023precise,
  title={Precise zero-shot dense retrieval without relevance labels},
  author={Gao, Luyu and Ma, Xueguang and Lin, Jimmy and Callan, Jamie},
  booktitle={Proceedings of the 61st Annual Meeting of the Association for Computational Linguistics (Volume 1: Long Papers)},
  pages={1762--1777},
  year={2023}
}

@article{huang2025atlaskv,
  title={AtlasKV: Augmenting LLMs with Billion-Scale Knowledge Graphs in 20GB VRAM},
  author={Huang, Haoyu and Tsang, Hong Ting and Bai, Jiaxin and Peng, Xi and Zhang, Gong and Song, Yangqiu},
  journal={arXiv preprint arXiv:2510.17934},
  year={2025}
}

@article{wei2025mlp,
  title={Mlp memory: A retriever-pretrained memory for large language models},
  author={Wei, Rubin and Cao, Jiaqi and Wang, Jiarui and Kai, Jushi and Guo, Qipeng and Zhou, Bowen and Lin, Zhouhan},
  journal={arXiv preprint arXiv:2508.01832},
  year={2025}
}

@article{fleshman2025lora,
  title={LoRA-Augmented Generation (LAG) for Knowledge-Intensive Language Tasks},
  author={Fleshman, William and Van Durme, Benjamin},
  journal={arXiv preprint arXiv:2507.05346},
  year={2025}
}

@article{lee2025gemini,
  title={Gemini embedding: Generalizable embeddings from gemini},
  author={Lee, Jinhyuk and Chen, Feiyang and Dua, Sahil and Cer, Daniel and Shanbhogue, Madhuri and Naim, Iftekhar and {\'A}brego, Gustavo Hern{\'a}ndez and Li, Zhe and Chen, Kaifeng and Vera, Henrique Schechter and others},
  journal={arXiv preprint arXiv:2503.07891},
  year={2025}
}

\appendix

\section{Theoretical Framework}
\label{sec:theory}

This section provides a principled view of \textbf{RING} as a \emph{search-free} approximation to the classical Retrieval-Augmented Generation (RAG) objective.
Our key message is: (i) classical RAG can be written as marginalizing a discrete latent \emph{evidence} variable over a corpus; (ii) RING parameterizes this latent retrieval distribution entirely within the model (router + knowledge expert), after injecting the corpus into $\Theta_{\mathrm{mem}}$ via CPT; and (iii) our CPT--SFT--RL pipeline can be interpreted as optimizing a variational lower bound and an evidence-selection policy.

\subsection{Problem Formulation: Discrete Latent Evidence for RAG}
Let the external knowledge corpus be $\mathcal{D}$, and let
\[
\mathcal{U}=\{u_z\}_{z=1}^{N}
\]
denote the set of \emph{memory units} derived from $\mathcal{D}$ after preprocessing (Section~\ref{CPT}): each $u_z$ is the JSON-style fragment
\texttt{\{title, content\}} (possibly truncated) used consistently in CPT/SFT/RL.
Given a query $x$ and target answer $y$, we model knowledge-intensive generation with a discrete latent evidence index $z\in\{1,\dots,N\}$:
\begin{equation}
p(y\mid x,\mathcal{U})
=
\sum_{z=1}^{N} p(z\mid x,\mathcal{U})\; \pgen(y\mid x, u_z).
\label{eq:app_rag_latent}
\end{equation}
Here $p(z\mid x,\mathcal{U})$ is the retrieval distribution, and
$\pgen(y\mid x,u_z)$ is the generator conditioned on evidence $u_z$~\cite{xu2025theory,xu2024unsupervised,duan2026skillattack}.

\paragraph{Explicit (search-based) RAG as truncated marginalization}
In conventional RAG, $p(z\mid x,\mathcal{U})$ is implemented by an external retriever (e.g., embedding similarity + ANN search) with score $s_\eta(x,u_z)$:
\[
p_\eta(z\mid x,\mathcal{U})
=
\frac{\exp(s_\eta(x,u_z))}{\sum_{j=1}^N \exp(s_\eta(x,u_j))}.
\]
Since $N$ is large, systems typically retrieve a top-$K$ set $\mathcal{T}_K(x)$ and renormalize:
\begin{equation}
p_{\mathrm{RAG}}(y\mid x,\mathcal{U})
\approx
\!\!\!\sum_{z\in\mathcal{T}_K(x)}\!\!\!
\tilde p_\eta(z\mid x,\mathcal{U})\, \pgen(y\mid x,u_z),
\label{eq:app_rag_topk}
\end{equation}
where $\tilde p_\eta(z)=p_\eta(z)/\sum_{z'\in\mathcal{T}_K(x)}p_\eta(z')$.

\subsection{RING Objective: Search-Free Parametric Retrieval}
RING eliminates external access to $\mathcal{U}$ at inference by \emph{(i)} injecting $\mathcal{U}$ into the knowledge expert parameters during CPT, and \emph{(ii)} learning a parametric retrieval policy during SFT/RL.

\begin{definition}[RING Latent Evidence Objective]
\label{def:app_RING_obj}
Let $\Theta=(\bar{\Theta},\Theta_{\mathrm{mem}},\phi)$, where $\bar{\Theta}$ denotes all frozen backbone parameters (including the base expert and attention blocks), $\Theta_{\mathrm{mem}}=\{\Wkd,\Wku\}_{\ell=1}^{L}$ are the trainable memory projections in the knowledge expert (Section~\ref{CPT}), and $\phi$ are the parameters of the router $r_\phi$ (and the Search Gate in $E_{\text{knw}}^{s}$).
We define the parametric joint model
\begin{equation}
p_{\Theta}(y,z\mid x)
=
p_{\phi}(z\mid x)\;\; \pgen(y\mid x,u_z;\Theta_{\mathrm{mem}},\bar{\Theta}),
\label{eq:app_joint}
\end{equation}
and the marginal
\begin{equation}
p_{\Theta}(y\mid x)=\sum_{z=1}^{N} p_{\phi}(z\mid x)\; \pgen(y\mid x,u_z;\Theta_{\mathrm{mem}},\bar{\Theta}).
\label{eq:app_marginal}
\end{equation}
At inference, RING does not query $\mathcal{U}$ externally; the dependence on $\mathcal{U}$ is realized only through the injected parameters $\Theta_{\mathrm{mem}}$.
\end{definition}

\paragraph{Connection to \texttt{<retrieval>} generation}
In the \emph{Search-then-Answer} paradigm (Section~\ref{sft}), RING explicitly generates a retrieval block that is intended to reproduce (one of) the memorized units $u_z$.
Thus, we can interpret $p_{\phi}(z\mid x)$ as the induced probability that the model generates unit $u_z$ inside \texttt{<retrieval>}...\texttt{</retrieval>} given the question $x$.

\paragraph{Top-1 (MAP) evidence approximation and an error bound}
In practice, RING often aims to retrieve a single best unit (top-1) before answering.
Define a score
\[
a_z \;=\; \log p_{\phi}(z\mid x) + \log \pgen(y\mid x,u_z;\Theta_{\mathrm{mem}},\bar{\Theta}),
\]
and let $z^\star=\arg\max_z a_z$.
Then log-sum-exp yields a standard approximation bound:
\begin{align}
0 \;\le\;
&\log \sum_{z=1}^{N} e^{a_z} - a_{z^\star} \nonumber\\
=\;
&\log\!\Big(1+\!\!\sum_{z\ne z^\star}\!\! e^{a_z-a_{z^\star}}\Big) \nonumber\\
\le\;
&\log\!\big(1+(N-1)e^{-\Delta}\big),
\label{eq:app_map_gap}
\end{align}
where $\Delta=a_{z^\star}-\max_{z\ne z^\star} a_z$ is the score gap.
Therefore, when the evidence posterior is peaked (large $\Delta$), the top-1 evidence approximation is provably close to the full marginal in Eq.~\eqref{eq:app_marginal}.

\subsection{Architecture View: Knowledge Expert as Parametric Key--Value Retrieval}
We now connect RING's knowledge expert to a differentiable retrieval mechanism.

\paragraph{Single-layer view}
Consider the knowledge expert MLP at layer $\ell$ applied to a hidden state $h\in\mathbb{R}^{d}$.
With the Memory Up/Down and Search Gate separation (Section~\ref{model_arch}), a convenient abstraction is the gated-MLP form (e.g., SiLU-gating):
\begin{align}
E^{(\ell)}_{\text{knw}}(h)
&=
\Wkd \Big(
\mathrm{SiLU}\big(\Wkg\, h\big) \nonumber \\
&\quad \odot\
\big(\Wku\, h\big)
\Big),
\label{eq:app_knw_gated_mlp}
\end{align}
where $\Wkd$ and $\Wku$ belong to $\Theta_{\mathrm{mem}}$ (trainable in CPT), while $\Wkg$ corresponds to the Search Gate parameters in $E_{\text{knw}}^{s}$ (trainable in SFT/RL).

\begin{proposition}[Gated-MLP as Key--Value Memory Readout]
\label{prop:app_kv_readout}
Let $d_{\mathrm{ff}}$ be the intermediate width of the MLP, and denote by
$k^{(\ell)}_i$ the $i$-th row of $\Wku$,
by $g^{(\ell)}_i$ the $i$-th row of $\Wkg$,
and by $v^{(\ell)}_i$ the $i$-th column of $\Wkd$.
Then Eq.~\eqref{eq:app_knw_gated_mlp} can be written as
\begin{align}
E^{(\ell)}_{\text{knw}}(h)
&=
\sum_{i=1}^{d_{\mathrm{ff}}}
\alpha^{(\ell)}_i(h)\; v^{(\ell)}_i, \nonumber\\
\alpha^{(\ell)}_i(h)
&=
\mathrm{SiLU}\big(\langle g^{(\ell)}_i,h\rangle\big)\cdot \langle k^{(\ell)}_i,h\rangle.
\label{eq:app_kv_form}
\end{align}
Hence each hidden representation $h$ acts as a \emph{query}, the rows of $W_{\text{knw},\text{up}}$ (modulated by the Search Gate) act as \emph{keys}, and the columns of $W_{\text{knw},\text{down}}$ act as \emph{values}.
\end{proposition}

\begin{proof}
Expand the matrix products in Eq.~\eqref{eq:app_knw_gated_mlp} coordinate-wise:
$W_{\text{knw},\text{down}}h$ and $W_{\text{knw},\text{gate}}h$ produce $d_{\mathrm{ff}}$
scalar activations; the elementwise product yields coefficients $\alpha_i(h)$, and the
final multiplication by $W_{\text{knw},\text{down}}$ sums columns weighted by these coefficients.
\end{proof}

\begin{lemma}[Approximate MIPS via Low-Temperature Softmax]
\label{lem:app_mips_softmax}
Let $s_i(h)$ be any scalar scores (e.g., $s_i(h)=\alpha_i(h)$ or an affine transform thereof).
Define $\pi_\beta(i\mid h)=\mathrm{softmax}_i(\beta s_i(h))$ and the readout
$R_\beta(h)=\sum_i \pi_\beta(i\mid h)\,v_i$. If the maximizer $i^\star=\arg\max_i s_i(h)$ is unique,
then $\pi_\beta(i^\star\mid h)\to 1$ and $R_\beta(h)\to v_{i^\star}$ as $\beta\to\infty$.
\end{lemma}

\begin{proof}
This is the standard low-temperature limit of softmax: probability mass concentrates
on the unique argmax as $\beta\to\infty$.
\end{proof}

\paragraph{Two-level sparsity: router vs. memory-slot selection}
RING has two complementary sparsity mechanisms:
(i) the router $r_\phi$ selects between experts $E_{\text{base}}$ and $E_{\text{knw}}$ via
$g^{(\ell)}_t$ (Section~3.1), controlling \emph{when} to consult injected knowledge;
(ii) within $E_{\text{knw}}$, the gated-MLP readout (Eq.~\eqref{eq:app_kv_form})
acts as a differentiable retrieval over a large number of implicit memory slots, controlling
\emph{what} knowledge is activated. This provides a parameter-space analogue of similarity search.

\subsection{Why RING Can Match or Exceed Explicit RAG}
We formalize the comparison as a hypothesis-class inclusion argument.

\begin{theorem}[Expressivity Dominance on the Query Domain]
\label{thm:app_containment}
Assume that after CPT, the injected parameters $\Theta_{\mathrm{mem}}$ store the evidence units
$\{u_z\}_{z=1}^{N}$ with sufficient fidelity, and that the router/Search Gate parameters $\phi$
have enough capacity to represent the same retrieval distribution as an external retriever on
the query domain of interest; i.e., for all relevant $x$,
\[
p_\phi(z\mid x)\approx \tilde p_\eta(z\mid x,\mathcal{U})
\quad \text{for } z\in\mathcal{T}_K(x),
\]
up to arbitrarily small error.
Then for any explicit top-$K$ RAG model in Eq.~\eqref{eq:app_rag_topk}, there exists a RING
instance in Eq.~\eqref{eq:app_marginal} that matches its conditional distribution $p(y\mid x)$
arbitrarily well on that domain (in total variation / KL, depending on the approximation metric).
\end{theorem}

\begin{proof}[Proof sketch]
Construct $p_\phi(z\mid x)$ to approximate the external truncated-retrieval distribution
$\tilde p_\eta(z\mid x,\mathcal{U})$ on $\mathcal{T}_K(x)$, and set the conditional generator
$\pgen(y\mid x,u_z)$ to match the same evidence-conditioned generator.
Since the evidence units are stored in $\Theta_{\mathrm{mem}}$, RING can reproduce the same
mixture in Eq.~\eqref{eq:app_rag_topk} without external search.
\end{proof}

\begin{corollary}[Best Achievable NLL is Not Worse (Under Capacity Assumptions)]
\label{cor:app_risk}
Let $\mathcal{R}(f)=\mathbb{E}_{(x,y)\sim P^\star}[-\log p_f(y\mid x)]$ be population NLL.
Under the assumptions of Theorem~\ref{thm:app_containment}, the optimal achievable risk of RING
is no worse than that of explicit top-$K$ RAG on the same domain:
\[
\inf_{\Theta}\mathcal{R}\big(p_{\Theta}\big)
\;\le\;
\inf_{\eta,\theta_{\mathrm{gen}}}\mathcal{R}\big(p_{\mathrm{RAG}}\big).
\]
\end{corollary}

\begin{remark}[When can RING be strictly better?]
Even with identical backbone generators, explicit RAG is constrained by external retrieval
form (e.g., embedding similarity), ANN approximation, top-$K$ truncation, and finite prompt/context
budget. RING jointly optimizes routing and generation and can learn task-aligned evidence selection
beyond embedding similarity, potentially achieving lower risk on distributions where similarity
search is a suboptimal proxy for selecting evidence.
\end{remark}

\subsection{Dual Causal Attention: Formal Mask and Loss Definitions}
\label{sec:dca-formal}
We give the full formal specification of the Dual Causal Attention (DCA) objective used in CPT (Section~\ref{CPT}).

Let a token sequence be $x_{1:T}=(x_1,\ldots,x_T)$ and consider a decoder layer with queries, keys, and values $Q,K,V\in\mathbb{R}^{T\times d}$. For a binary mask $M\in\{0,1\}^{T\times T}$, the masked attention is
\begin{align*}
\mathrm{Attn}(Q,K,V;M)
&=\mathrm{softmax}\!\left(\frac{QK^\top}{\sqrt d}+\log M\right)V, \\
&\text{with }\log 0:=-\infty,\ \log 1:=0.
\end{align*}
For $\alpha\in\{1,\tfrac12,\tfrac14\}$, let $K_\alpha=\lfloor \alpha T\rfloor$ and define the \emph{target block} $B_\alpha=\{1,\ldots,K_\alpha\}$ and its complement $C_\alpha=\{K_\alpha+1,\ldots,T\}$. The masks $M^{(\alpha)}\in\{0,1\}^{T\times T}$ specify which source positions $j$ are visible when predicting position $i$. For $\alpha=1$ the standard causal mask is $M^{(1)}_{i,j}=\mathbf{1}[\,j\le i\,]$; for $\alpha\in\{\tfrac12,\tfrac14\}$ the dual causal mask is
\begin{align}
M^{(\alpha)}_{i,j}&=
\begin{cases}
1, & i\!\in\! B_\alpha,\, j\!\le\! i \text{ or } j\!\in\! C_\alpha,\\[2pt]
\mathbf{1}[\,j\le i\,], & i\in C_\alpha.
\end{cases}
\label{eq:dca-mask}
\end{align}
Thus, when predicting tokens inside $B_\alpha$, the model preserves causality within $B_\alpha$ (no look-ahead to future targets in the block) while attending to all of $C_\alpha$ as backward evidence; positions in $C_\alpha$ retain the standard causal constraint.

For each $\alpha$, the negative log-likelihood over the corresponding target block is
\begin{equation}
\mathcal{L}^{(\alpha)}=
\frac{1}{|B_\alpha|}
\sum_{t\in B_\alpha}
-\log p_\theta\!\left(x_t \,\middle|\, x_{1:T};\, M^{(\alpha)}\right),
\label{eq:dca-loss}
\end{equation}
where $p_\theta(\cdot\,|\,x_{1:T};\,M^{(\alpha)})$ is obtained by running the network on the full sequence with mask $M^{(\alpha)}$ and reading logits only at positions $t\in B_\alpha$. The three losses are combined as
\begin{align}
\mathcal{L}_{\mathrm{DCA}}=
\lambda_{1}\,\mathcal{L}^{(1)}+
\lambda_{2}\,\mathcal{L}^{(1/2)}+
\lambda_{3}\,\mathcal{L}^{(1/4)}.
\label{eq:dca-total}
\end{align}
In practice, the same training example is passed three times (once per mask), gradients are accumulated according to Eq.~\eqref{eq:dca-loss} and mixed via Eq.~\eqref{eq:dca-total}; at inference, only $M^{(1)}$ is used.

\subsection{Training as Variational Inference: Linking CPT--SFT--RL to an ELBO}
We derive an evidence lower bound (ELBO) for the discrete latent model in Eq.~\eqref{eq:app_joint}.

\paragraph{ELBO with a variational posterior}
Introduce a variational posterior $q(z\mid x,y)$.
Then
\begin{align}
&\log p_{\Theta}(y\mid x) \nonumber\\
&\;=\;
\log \sum_{z=1}^{N} q(z\mid x,y)\,
\frac{p_{\phi}(z\mid x)\,\pgen(y\mid x,u_z)}{q(z\mid x,y)}
\nonumber\\
&\;\ge\;
\underbrace{\mathbb{E}_{q}\!\big[\log \pgen(y\mid x,u_z)\big]}_{\text{reconstruction}}
\nonumber\\
&\quad-\;
\underbrace{\mathrm{KL}\big(q(z\mid x,y)\,\|\,p_{\phi}(z\mid x)\big)}_{\text{posterior--policy alignment}}.
\label{eq:app_elbo}
\end{align}
At test time, RING uses the policy $p_\phi(z\mid x)$ (or its top-1 approximation),
while $q(z\mid x,y)$ is only a training-time analysis tool.

\paragraph{Step 1 (CPT): fitting the parametric memory $\Theta_{\mathrm{mem}}$}
During CPT, we force routing to the knowledge expert and update only
$\Theta_{\mathrm{mem}}=\{\Wku,\Wkd\}$,
optimizing the DCA objective on $\mathcal{D}_{\mathrm{CPT}}$ (Section~\ref{CPT}).
This stage increases the capacity of $\pgen(\cdot\mid x,u_z)$ to represent and
reproduce evidence units $u_z$ through the knowledge expert, preparing the model for subsequent
evidence selection.

\paragraph{A principled view of DCA}
DCA provides additional right-context evidence for predicting tokens in a target block while retaining
autoregressive generation at inference. The key theoretical property we rely on is monotonicity of
optimal conditional cross-entropy with respect to conditioning sets.

\begin{proposition}[DCA Does Not Worsen the Optimal Conditional Cross-Entropy]
\label{prop:app_dca_monotone}
Consider predicting a target block conditioned on its complement.
Any training objective that allows tokens in the target block to condition on a \emph{superset} of
information (as DCA does for the target block) has an optimal achievable negative log-likelihood
no larger than that of standard causal conditioning, because the model can always learn to ignore
additional context. Consequently, DCA can provide a tighter (not looser) upper bound on the
conditional entropy of the target block given its complement, improving the information available
for storing bidirectional dependencies into $\Theta_{\mathrm{mem}}$ during CPT.
\end{proposition}

\paragraph{Step 2 (SFT): supervised posterior collapse via \texttt{<retrieval>}}
In SFT, each instruction sample provides a gold evidence unit $u^\star=u_{z^\star}$ obtained by exact
string match (Section~\ref{sft}). Setting $q(z\mid x,y)=\delta(z=z^\star)$ in Eq.~\eqref{eq:app_elbo}
yields the supervised objective
\[
\max\ \log p_\phi(z^\star\mid x) + \log \pgen(y\mid x,u_{z^\star}),
\]
which corresponds exactly to teacher-forcing on the \texttt{<retrieval>} block (learning evidence selection)
and the \texttt{<answer>} block (learning answer synthesis given evidence). Freezing $\Theta_{\mathrm{mem}}$
during SFT prevents overwriting the memorized corpus and focuses learning on $r_\phi$ and the Search Gate.

\paragraph{Step 3 (RL): learning a generalizable evidence-selection policy}
After SFT, we apply RLVR to further align the evidence-selection behavior with task success.
Because different implementations may define ``rewards'' or ``penalties'' with opposite signs,
we present the RL objective in a sign-agnostic form.

Let $\ell_{\mathrm{fmt}},\ell_{\mathrm{ans}},\ell_{\mathrm{dense}}\in[0,1]$ denote \emph{penalties}
(lower is better) for (i) format compliance, (ii) verified answer correctness, and (iii) evidence match
(e.g., overlap/LCS-based signals on \texttt{title/content}), respectively. Define the total penalty
\[
\ell_{\mathrm{total}}
=
\lambda_{\mathrm{fmt}}\ell_{\mathrm{fmt}}
+
\lambda_{\mathrm{ans}}\ell_{\mathrm{ans}}
+
\lambda_{\mathrm{dense}}\ell_{\mathrm{dense}},
\]
and maximize the \emph{negative} penalty (equivalently, minimize expected penalty):
\begin{align*}
&\max_{\phi}\ \Ezp\!\big[-\ell_{\mathrm{total}}(z,x,y)\big] \\
\Longleftrightarrow\;
&\min_{\phi}\ \Ezp\!\big[\ell_{\mathrm{total}}(z,x,y)\big].
\end{align*}
This formulation is compatible with either convention (reward-as-1-for-correct or penalty-as-1-for-incorrect).

\begin{lemma}[Dense Penalties Improve Credit Assignment for Retrieval]
\label{lem:app_dense_credit}
Suppose $\ell_{\mathrm{ans}}$ is sparse (many actions $z$ receive identical penalties because
answer correctness is hard to achieve), while $\ell_{\mathrm{dense}}(z,z^\star)$ varies smoothly with the
match quality between the sampled evidence $z$ and the gold evidence $z^\star$.
Then optimizing $\mathbb{E}[\ell_{\mathrm{dense}}]$ provides informative gradients for the policy
$p_\phi(z\mid x)$ even when $\ell_{\mathrm{ans}}$ provides little signal, accelerating the concentration
of $p_\phi(\cdot\mid x)$ around high-quality evidence choices.
\end{lemma}

\begin{remark}[Relation to the ELBO]
When the dense signal correlates with evidence correctness, minimizing expected penalty can be viewed
as a practical surrogate for reducing the mismatch between the policy $p_\phi(z\mid x)$ and the
ideal posterior over evidence. This complements the supervised alignment in SFT and improves
generalization of evidence selection beyond exact-match supervision.
\end{remark}

\section{Datasets}
\label{datasets}
\textbf{Memorization Set (CPT):} We collect the raw news corpus and segment it into fine-grained units, resulting in approximately 1.6 million fragments. This set serves as the external knowledge base $\mathcal{D}_{\mathrm{CPT}}$ that needs to be injected. 

\textbf{Instruction Set (SFT):} Based on the CPT corpus, we generate 317k instruction-following samples using the pipeline described in Section~\ref{sft}. These samples are designed to teach the model the ``Search-then-Answer'' paradigm. The detailed construction recipe is given below.

\textbf{Alignment Set (RL):} We construct a further 200k samples specifically for the RL stage.

\subsection{Instruction Fine-tuning Data Construction Details}
\label{sft-data-details}
We expand the main-text description of the SFT data pipeline.

Let $\mathcal{D}$ be the external corpus to be injected and $\mathcal{D}_{\mathrm{CPT}}$ the preprocessed set of memory units used during CPT (each unit $u$ is the JSON-style \texttt{\{title, content\}} fragment defined in Section~\ref{CPT}). We subsample documents $\mathcal{S}\subseteq\mathcal{D}$ with rate $\rho\in(0,1]$, a tunable hyperparameter trading off coverage against training cost.

\paragraph{Triple Synthesis} We employ \texttt{qwen3-235b-a22b-instruct-2507} to produce $n$ supervision triples $\{(q_i, a_i, r_i)\}_{i=1}^{n}$ per sampled document $d\in\mathcal{S}$, where $q_i$ is an instruction-style question, $a_i$ is a concise answer grounded in $d$, and $r_i$ is an explicit \emph{reference span} from $d$ that justifies $a_i$. We control $n$ to balance coverage and cost.

\paragraph{Reference--Memory Alignment} Let $\mathcal{U}=\{u_j\}$ be the set of memory units in $\mathcal{D}_{\mathrm{CPT}}$. We align each reference $r_i$ to its corresponding memory unit $u\in\mathcal{D}_{\mathrm{CPT}}$ via exact string match: the pair is accepted only if there exists a $u_i^\star\in\mathcal{U}$ that matches $r_i$; otherwise the triple is discarded. The accepted $u_i^\star$ is truncated to at most $256$ tokens to bound generation cost.

\paragraph{Target Template} Each retained triple is converted into a \emph{search-then-answer} target string that the model must autoregressively produce given the question. Concretely, \texttt{Instruction}: $x^{\mathrm{in}}=q_i$; \texttt{Response}: $y^{\mathrm{tg}}=$ \texttt{I can retrieve this passage to answer the question} \texttt{<retrieval>}\,$u_i^\star$\,\texttt{</retrieval>} \texttt{<answer>}\,$a_i$\,\texttt{</answer>}. The \texttt{<retrieval>} block supervises generative retrieval, and the \texttt{<answer>} block supervises final answering.

\paragraph{Retrieve-only-when-needed} To prevent the model from always emitting \texttt{<retrieval>}, we mix in no-retrieval instructions from a generic instruction corpus $\mathcal{G}$ (math, coding, reasoning), for which the response does not follow the Search-then-Answer template. The final SFT set is
$\mathcal{I}=\{(q, u^\star, a)\}\cup\{(q, \varnothing, a)\}$,
stratified into train/dev/test by document ID to prevent leakage. During training we interleave $\mathcal{I}$ with a general instruction set to stabilize early learning and improve instruction following.

\paragraph{Loss} Given input $x^{\mathrm{in}}$ and target $y^{\mathrm{tg}}$, the parameters $\theta$ are optimized with standard cross-entropy:
\[
\mathcal{L}_{\mathrm{SFT}}(\theta)=
\mathbb{E}_{(q,u^\star,a)\sim \mathcal{I}}\Big[-\!\!\sum_{t=1}^{|y^{\mathrm{tg}}|}\!\log p_\theta\!\big(y^{\mathrm{tg}}_t\mid x^{\mathrm{in}}, y^{\mathrm{tg}}_{<t}\big)\Big],
\]
where the presence of the \texttt{<retrieval>} block directly supervises the decision to retrieve.

\section{Experimental Setting} \label{exp_setting}

\subsection{Inference Pipelines}
To ensure a rigorous and fair comparison between the parametric and external retrieval paradigms, we strictly utilize the identical knowledge corpus for all methods.
For conventional search-based RAG baselines, the inference pipeline follows a standard retrieve-then-generate workflow: given a query, the system searches the external vector index to retrieve the relevant documents. This retrieved context is then provided to the LLM to guide answer generation.
In contrast, our RING operates in a search-free manner. The model is presented solely with the query, and the retrieval process is internalized. The model directly generates the answer by leveraging the knowledge encoded within its parameters (specifically the Knowledge Expert), completely eliminating the interaction with external search engines during inference.

\subsection{Evaluation Metrics}
We assess performance based on factual accuracy. The ground truth answers in our benchmark are pre-processed to be concise and unambiguous. To scale the evaluation, we employ a model-based judge using \texttt{DeepSeek-R1-0528}. The judge compares the generated response against the ground truth and assigns a binary score (1 for a match, 0 otherwise). We report the final Accuracy (\%) averaged across the test set.

\subsection{Baseline Details}
\label{baselines-details}
We provide the full descriptions of all baseline categories compared in Section~\ref{sec:main_results}.

\paragraph{Conventional Search-based RAG (Top-1)} This category represents the standard retrieve-and-read RAG. They rely on an external search engine to fetch relevant context, which is then concatenated to the input of the frozen base LLM. We employ a variety of state-of-the-art (SOTA) embedding models to perform retrieval, ensuring a strong upper bound: BGE-M3~\cite{chen2024bge}, OpenAI-text-embedding-3-large, Gemini-Embedding~\cite{lee2025gemini}, and Qwen3-Embedding~\cite{zhang2025qwen3}.

\paragraph{Stronger RAG Pipelines} To address the concern that a top-1 retrieve-and-read setup underestimates modern RAG, we additionally evaluate: \textbf{(i) Top-$k$ retrieval} with $k\in\{3,5,10\}$ using the strongest embedding (Qwen3-8B-Emb); \textbf{(ii) retrieval + reranking}, where we retrieve top-10 candidates and rerank them with a cross-encoder (BGE-Reranker-v2-m3 and Qwen3-Reranker-4B) to select the final top-3; \textbf{(iii) HyDE}~\cite{gao2023precise}, which generates a hypothetical answer to rewrite the query before retrieval; and \textbf{(iv) Query Rewrite}, in which we prompt the base LLM to paraphrase the question prior to retrieval.

\paragraph{Parametric Knowledge Injection and Memorization} These baselines aim to inject the external corpus $\mathcal{D}$ directly into the model parameters. \textbf{(1) Full Tuning}: standard full-parameter CPT on the raw corpus $\mathcal{D}_\text{CPT}$ followed by instruction tuning on large-scale Q\&A pairs constructed from $\mathcal{D}$. \textbf{(2) LoRA Tuning}: same training set as Full Tuning but applied via LoRA. \textbf{(3) Modular Injection}: methods that mount knowledge into the LLM's computational flow, including KBLAM~\cite{wang2025kblam} and SR-KI~\cite{yu2025sr}. \textbf{(4) Retriever-Imitating Parametric Memory}: three recent methods closest to RING in spirit---LAG~\cite{fleshman2025lora} attaches per-document LoRA adapters and selects them at inference via data-free routing; MLP Memory~\cite{wei2025mlp} pretrains a dedicated MLP to imitate a $k$NN retriever's distribution; AtlasKV~\cite{huang2025atlaskv} stores billion-scale knowledge as compressed key--value caches augmented into the LLM. A significant drawback of (3) and partially of (4) is that GPU memory usage grows with the size of the injected corpus. Compared with these baselines, RING (i) keeps a separate Basic Expert to prevent forgetting, (ii) injects the corpus via bidirectional DCA rather than fitting a retriever's logits or storing raw KV slabs, and (iii) learns \emph{when} and \emph{what} to retrieve through RL with dense rewards rather than via a fixed similarity-based mechanism.

\subsection{Implementation Details} \label{imple}
We utilize Qwen3-8B and Qwen3-14B as our backbone foundation model. Our training pipeline consists of three sequential stages. We use the AdamW optimizer with a cosine learning rate scheduler for all stages. In CPT, we train the model on the constructed News-2025 corpus. During this stage, we freeze all parameters except for the Memory Down and Memory Up projections of the Knowledge Expert. We employ the proposed Dual Causal Attention with mask ratios $\alpha \in \{1, 1/2, 1/4\}$ and loss weights $\lambda_1=1.0$, $\lambda_2=0.3$, and $\lambda_3=0.1$. The max sequence length is set to 4096, the learning rate is 2e-5 and epoch is 8. In SFT, we utilize the Search-then-Answer instruction set mixed with general instructions at a ratio of 1:1. Only Router and Search Gate are trainable. Retrieved fragments are truncated to 256 tokens. The max sequence length is set to 4096, the learning rate is 2e-5 and epoch is 4 with early-exit. In RL, we use ms-swift as the framework. \(\lambda_{\text{a}}\), \(\lambda_{\text{b}}\), and \(\lambda_{\text{c}}\) are 0.8, 0.8, 1.0 respectively. \(\lambda_{\text{format}}\), \(\lambda_{\text{answer}}\), and \(\lambda_{\text{search}}\) are 0.1, 1.0, 1.0 respectively. All models are implemented using PyTorch and trained on 32 $\times$ NVIDIA H800 GPUs. We utilize DeepSpeed Zero-3 for memory optimization. The learning rate is 1e-6.

\section{AI Use Statement}
We used AI-based writing assistance only for language polishing, including grammar correction,
wording refinement, and improving readability. AI tools were not used for formulating the research
ideas, designing the method, conducting experiments, analyzing results, or generating substantive
scientific content. All technical claims, experimental results, and conclusions were produced and
verified by the authors.

\end{document}